%% file: main_camera_ready.tex
\documentclass[manuscript]{jair}

\setcopyright{cc}
\copyrightyear{2026}
\acmDOI{10.1613/jair.1.19233}

\JAIRAE{Bettina Koenighofer}
\JAIRTrack{Integration of Logical Constraints in Deep Learning} %
\acmVolume{85}
\acmArticle{25}
\acmMonth{03}
\acmYear{2026}

\RequirePackage[
  datamodel=acmdatamodel,
  style=acmauthoryear,
  backend=biber,
  giveninits=true,
  uniquename=init
  ]{biblatex}

\renewcommand*{\bibopenbracket}{(}
\renewcommand*{\bibclosebracket}{)}

 \usepackage{tikz}
\usetikzlibrary{automata, arrows.meta, positioning}
\usepackage{local_macros}

\usetikzlibrary{matrix}

\definecolor{mplblue}{HTML}{1F77B4}
\definecolor{mplorange}{HTML}{FF7F0E}
\definecolor{mplgreen}{HTML}{2CA02C}
\usepackage[T1]{fontenc}
\usepackage{wrapfig}

\usepackage{amsmath}
\usepackage{algorithm}
\usepackage{algorithmic}
\theoremstyle{plain}
\newtheorem{assumption}{Assumption}
\newtheorem{theorem}{Theorem}[section]

\newtheorem{lemma}[theorem]{Lemma}

\theoremstyle{definition}
\newtheorem{definition}[theorem]{Definition}
\newtheorem{example}[theorem]{Example}

\newtheorem{problem}{Problem}
\theoremstyle{remark}
\newtheorem{remark}[theorem]{Remark}
\numberwithin{equation}{section}

\usepackage{subcaption}

\setcopyright{cc}
\begin{document}

\title[Average-Reward RL for Omega-Regular and Mean-Payoff Objectives]{Average Reward Reinforcement Learning for Omega-Regular and Mean-Payoff Objectives}

\author{Milad Kazemi}
\orcid{0000-0003-2716-2758}
\email{milad.kazemi@kcl.ac.uk}
\affiliation{%
  \institution{King’s College London}
  \city{London}
  \country{UK}
}

\author{Mateo Perez}
\orcid{0000-0003-4220-3212}
\email{Mateo.Perez@colorado.edu}
\affiliation{%
  \institution{University of Colorado Boulder}
  \city{Colorado}
  \state{Boulder}
  \country{USA}
}

\author{Fabio Somenzi}
\orcid{0000-0002-2085-2003}
\email{fabio@colorado.edu}
\affiliation{%
  \institution{University of Colorado Boulder}
  \city{Colorado}
  \state{Boulder}
  \country{USA}
}

\author{Sadegh Soudjani}
\orcid{0000-0003-1922-6678}
\email{sadegh@mpi-sws.org}
\affiliation{%
  \institution{University of Birmingham, UK, and MPI-SWS}
  \city{Kaiserslautern}
  \country{Germany}
}

\author{Ashutosh Trivedi}
\orcid{0000-0001-9346-0126}
\email{Ashutosh.Trivedi@colorado.edu}
\affiliation{%
  \institution{University of Colorado Boulder}
  \city{Colorado}
  \state{Boulder}
  \country{USA}
}

\author{Alvaro Velasquez}
\orcid{0000-0001-6757-105X}
\email{Alvaro.Velasquez@colorado.edu}
\affiliation{%
  \institution{University of Colorado Boulder}
  \city{Colorado}
  \state{Boulder}
  \country{USA}
}

\renewcommand{\shortauthors}{Kazemi, Perez, Somenzi, Soudjani, Trivedi \& Velasquez}

\begin{abstract}
Recent advances in reinforcement learning (RL) have renewed focus on the design of reward functions that shape agent behavior. 
Manually crafting such functions is often tedious and error-prone. 
A more principled alternative is to specify behavioral requirements using a formal, unambiguous language that can be automatically translated into a reward function. 
Omega-regular languages are a natural choice for this purpose, given their established role in formal verification and synthesis.
However, existing approaches using omega-regular specifications typically rely on discounted reward RL in an episodic setting, where the environment is periodically reset to an initial state during learning. 
This setup is misaligned with the semantics of omega-regular specifications, which describe properties over infinite behavior traces. 
In such cases, the average reward criterion and the continuing setting---where the agent interacts with the environment over a single, uninterrupted lifetime---are more appropriate.

To address the challenges of infinite-horizon, continuing tasks, we restrict our focus to the subclass of omega-regular languages known as absolute liveness specifications. 
These specifications cannot be violated by any finite prefix of the agent’s behavior, aligning naturally with the continuing setting.
We present the first model-free RL framework that translates absolute liveness specifications to average-reward objectives. 
In contrast to prior work, our approach enables learning in communicating Markov Decision Processes without episodic resetting. We further introduce a reward structure for lexicographic multi-objective optimization, where the goal is to maximize an external average-reward objective among the policies that also maximize the satisfaction probability of a given absolute liveness omega-regular specification.
Our method guarantees convergence in unknown communicating MDPs and supports on-the-fly reductions that do not require full knowledge of the environment, thus enabling model-free RL. Empirical results across various benchmarks demonstrate that our average-reward approach in the continuing setting is more effective than competing methods based on discounting.     
\end{abstract}

\received{23 May 2025}
\received[accepted]{21 January 2026}

\maketitle

\section{Introduction}
\input{intro}

\section{Preliminaries and Problem Statement}
\label{sec:preliminaries}

This section presents the foundational concepts and formalizes the problem setting. 
We begin with Markov Decision Processes (MDPs) as the underlying model of the environment in Subsection~\ref{subsec:MDP}. In Subsection~\ref{subsec:quantitative}, we introduce probabilistic reward machines, which serve as quantitative objectives for RL-based policy synthesis. Subsection~\ref{subsec:qualitative} covers $\omega$-regular specifications as qualitative objectives for policy synthesis. We then formulate the formal problem statements in the context of reinforcement learning for continuing tasks in Subsection~\ref{subsec:RL_continual}, and conclude by discussing the challenges of average-reward RL in Subsection~\ref{subsec:challenges}.

\input{prelims}

\section{Average-Reward RL for Qualitative Objectives}
\label{sec:result}
\input{methodology}

\section{Average-Reward RL for Lexicographic Objectives}
\label{sec:morl}
\input{multiobjective}

\section{Experimental Results}
\label{sec:case-studies}

\input{exp}

\section{Conclusion}
\label{sec:conclusion}
\input{conclusion}

\begin{acks}
The research of S. Soudjani is supported by the following grants: EIC 101070802 and ERC 101089047.
This research was also supported in part by the National Science Foundation (NSF) through CAREER Award CCF-2146563. Ashutosh Trivedi holds the position of Royal Society Wolfson Visiting Fellow and gratefully acknowledges the support of the Wolfson Foundation and the Royal Society for this fellowship.

\end{acks}

\printbibliography
\appendix

\end{document}

%% file: intro.tex
Reinforcement learning (RL)~\cite{Sutton18} is a foundational framework for sequential decision making under uncertainty, where agents learn to optimize their behavior through trial-and-error interaction with the environment. While RL has achieved superhuman performance in domains such as board games~\cite{AlphaGo, AlphaGoZero, AlphaZero, MuZero}, robotics~\cite{zhao2020sim,tang2025deep}, and resource optimization~\cite{mirhoseini2021graph,mao2016resource}, much of this success has been in \emph{episodic tasks}, i.e., settings where interactions naturally reset after a finite sequence, and learning proceeds by leveraging repeated experience across episode\Milad{s}.
In contrast, many real-world applications---such as autonomous monitoring, industrial process control, or mission planning---are inherently continuing tasks, where the agent must learn over a single, unending interaction with its environment. 
For such tasks, short-term rewards often fail to capture the long-term behavioral goals of interest. Instead, it is more natural to specify objectives in terms of structured, temporally extended criteria.
Formal languages, particularly \emph{$\omega$-regular languages} recognized by automata over infinite traces~\cite{BK08}, provide a rich and precise way to express such long-term goals. 
This has led to increasing interest in integrating formal methods with RL for the synthesis of correct-by-construction policies~\cite{belta2019formal, jagtap2020formal, majumdar2021symbolic}, enabling logic-based specifications to guide data-driven learning.
This paper explores how \emph{model-free average-reward RL}, a formulation especially suited for continuing tasks~\cite{mahadevan1996average, wan2021learning}, can be used to synthesize policies that satisfy $\omega$-regular objectives. 

\subsection{Why $\omega$-Regular Objectives?}
While many successful applications of RL rely on reward functions that depend only on the agent’s current state and action, it is often necessary---or more natural---to account for the agent’s history when defining learning objectives. 
This becomes especially important in settings with sparse rewards~\cite{adviceSAT}, partial observability~\cite{toro2019learning}, or temporally extended goals~\cite{camacho2019ltl}. In such cases, the agent’s goal is better expressed as a sequence of desirable or undesirable behaviors rather than as immediate outcomes.
Formal language structures---particularly automata over sequences of observations---have emerged as a powerful tool for specifying these non-Markovian objectives. 
Rooted in formal verification~\cite{vardi1985automatic}, this approach enables the use of automata with a variety of acceptance conditions to capture complex reward structures and behavioral constraints. In some cases, even natural language objectives can be systematically translated into automata representations~\cite{syntacticSugar}.

This line of work has gained significant traction in the formal synthesis of control policies~\cite{belta2019formal, jagtap2020formal, majumdar2021symbolic,lavaei2020formal}, where developers define high-level goals in a formal language and synthesis algorithms compute correct-by-construction policies---eliminating the need for manual, error-prone reward engineering~\cite{kress2018synthesis}. 
In this paper, we adopt this paradigm to synthesize policies using model-free RL for a class of non-Markovian objectives defined by $\omega$-regular languages~\cite{BK08}. 
They capture infinite sequences of semantically meaningful observations and provide a rich formalism for specifying long-term behavioral goals.

To operationalize such specifications within RL, we introduce nondeterministic reward machines, which encode the reward structure of $\omega$-regular automata. By constructing a product between the agent’s environment---modeled as a Markov Decision Process (MDP)---and the automaton, we reduce the synthesis problem to a Markovian task that can be addressed using standard model-free algorithms. This builds on the broader use of reward machines~\cite{icarte2018using}, which represent reward functions as automata over histories of semantically meaningful events. These machines provide a structured and interpretable way to incorporate temporal logic into RL and enable the transformation of non-Markovian objectives into equivalent Markovian formulations over an augmented state space~\cite{gaon2020reinforcement, XuWuNeiderTopcu21}. This transformation allows existing RL techniques to be applied to long-horizon, history-dependent tasks without modification.

\subsection{Limitations of Discounted Reward in Continuing Tasks.} Despite the widespread use of discounted reward formulations in RL, they are often ill-suited for continuing tasks. The discounted return, defined as the sum of future rewards scaled by a discount factor, ensures mathematical tractability by bounding the total return over infinite horizons. 
However, this formulation inherently prioritizes short-term gains over long-term performance, which misaligns with the objectives of many continuing systems. In practice, achieving a suitable approximation of long-run behavior requires setting the discount factor very close to one. Yet, doing so weakens the contraction properties of learning algorithms, leading to slow convergence and potential instability.

These issues are particularly pronounced in continuing environments, where no natural episode boundaries exist. While discounted RL has shown remarkable success in episodic domains~\cite{Mnih15,AlphaGo,mirhoseini2021graph}, its solutions are sensitive to initial state distributions—an undesirable property in continuing settings where learning should be state-agnostic over the long run. Furthermore, discounted formulations in continuing tasks are often incompatible with function approximation techniques commonly used in large-scale RL. In~\cite{Naik2019DiscountedRL}, the authors argue that in continuing tasks, discounted RL is fundamentally incompatible with function approximation and does not constitute a proper optimization problem. Unlike tabular settings where any policy is representable, most practical RL problems have state or action spaces that are too large for tables. This necessitates function approximation to find the best representable policy. However, without a well-posed objective function, there is often no representable policy that is unambiguously superior to all others \cite{Naik2019DiscountedRL}.

Discounted RL algorithms for continuing tasks often rely on increasing the discount factor close to one to approximate the Blackwell-optimal policy, and they will succeed particularly when action gaps are large~\cite{bellemare2016increasing}. However, two major obstacles hinder the use of discounted-value algorithms as proxies for average-value algorithms in practice. First, most algorithms that learn discounted value functions become increasingly unstable as the discount factor approaches one. Second, estimating the required critical discount factor is difficult since it is tied to the environment's unknown dynamics~\cite{Naik2019DiscountedRL}.
Finally, the complexity of representing the optimal solution for long-horizon problems is studied in \cite{lehnert2018value} by isolating the representation challenge from the learning challenge. The authors show that function approximation methods may fail if action gaps collapse, which can occur over long planning horizons. The approximator may lack the necessary ``resolution'' to distinguish the optimal action from sub-optimal ones.

Given the growing reliance on large neural architectures in modern RL, the limitations of discounted RL motivate the need for alternative formulations aligned with the demands of continuing tasks.
These limitations motivate the need to explore average-reward RL as a more principled alternative in continuing settings.

\subsection{Average-Reward RL}
A natural alternative to discounting in continuing environments is to optimize the agent’s average reward over time. 
This formulation aligns more directly with long-term performance objectives, making it particularly well-suited for tasks without episodic resets. However, average-reward RL introduces unique challenges. Unlike discounted RL—where the discount factor ensures convergence through a built-in contraction—average-reward methods rely on structural properties of the underlying MDP.
In particular, the convergence of model-free average-reward algorithms~\cite{mahadevan1996average, wei2020model, wan2021learning} typically requires the MDP to be communicating, meaning every state is reachable from every other state under some policy.
This communicating assumption is especially important---and potentially problematic---when synthesizing policies for $\omega$-regular specifications. 
The synthesis process involves constructing a product MDP between the agent’s environment and the automaton representing the specification. While the original MDP may satisfy the communicating property, the product MDP may not, creating a barrier to applying average-reward RL directly in such settings.

\subsection{Absolute Liveness Properties.}
Absolute liveness specifications form an important subclass of $\omega$-regular specifications and are prevalent throughout the temporal property hierarchy~\citep{alpern1985defining,manna1990hierarchy}.
These specifications are prefix-independent—their satisfaction is unaffected by the addition of finite prefixes to an infinite word—making them well-suited for continuing tasks.
Moreover, when the underlying MDP is communicating, absolute liveness specifications are satisfied with probability zero or one.

A key observation is that many temporally extended objectives, including those expressed in Linear Temporal Logic (LTL), can be framed in terms of absolute liveness semantics.
In particular, a satisfiable LTL formula $\varphi$ is an absolute liveness property if it is expressively equivalent to $\mathbf{F}\varphi$, where $\mathbf{F}\varphi$ denotes eventual satisfaction of $\varphi$.
Thus, the eventual-satisfaction semantics of an arbitrary $\omega$-regular or LTL specification $\varphi$ can be captured by the absolute liveness property $\mathbf{F}\varphi$, motivating their use in long-horizon learning tasks.

\subsection{Contributions.}
The main contribution of this paper is an average-reward model-free reinforcement learning algorithm for synthesizing policies that satisfy a given absolute liveness $\omega$-regular specification. Our approach preserves the communicating property in the product construction, thereby enabling the learning of optimal policies without requiring episodic resetting. Although we assume the underlying MDP is communicating, a naive synchronization of the MDP with the automaton does not, in general, preserve this property. To address this, we introduce a reward machine construction and an augmented specification that guarantee the synchronized product MDP remains communicating.
This paper makes the following primary contributions:
\begin{enumerate}
\item We present average-reward model-free RL algorithms for:
\begin{enumerate}
\item synthesizing policies that satisfy a given \emph{absolute liveness} $\omega$-regular specification, and
\item solving lexicographic multi-objective optimizations in which the goal is to maximize a mean-payoff objective among the set of satisfying policies.
\end{enumerate}
\item We study the structure of $\omega$-regular specifications, including safety, liveness, absolute liveness, fairness, and stability. We introduce novel automata-theoretic characterizations of \emph{absolute liveness} (Lemma~\ref{th:st-containment}) and \emph{stable} specifications (Lemma~\ref{th:al-containment}), which may be of independent interest.

\item We analyze the convergence of our algorithms and demonstrate that naive synchronization of a communicating MDP with a specification automaton can break the communicating property. To address this, we construct \emph{reward machines} that preserve communication, enabling optimal learning without episodic resets.

\item Our work is the \emph{first to establish a formal reduction from absolute liveness $\omega$-regular objectives to average-reward RL} in both single- and multi-objective settings with convergence guarantees.

\item We generalize our prior results from communicating MDPs~\cite{kazemi2022translating} to the broader class of \emph{weakly communicating MDPs}, thereby 
extending the applicability of our framework to environments where some states may not be mutually reachable, yet still support meaningful long-run average-reward optimization.
\item We implement our proposed reduction using Differential Q-learning and evaluate it on a suite of communicating MDPs with absolute liveness specifications. Unlike prior methods, our approach composes the product MDP on-the-fly without requiring episodic resetting. Empirical results show that our method reliably converges to optimal policies in the continuing setting, even when prior approaches fail due to non-communicating product MDPs. In the episodic setting, our approach remains competitive in training time while requiring no resets, highlighting its robustness and practical effectiveness.
\end{enumerate}

To the best of our knowledge, this is the first framework that provides a formal translation from absolute liveness $\omega$-regular objectives to average-reward RL with convergence guarantees. Preliminary results were reported at the AAMAS conference~\cite{kazemi2022translating}; the current paper extends that work with complete proofs, generalizations (weakly communicating MDPs), and novel results on multi-objective synthesis.

\subsection{Organization.}
The paper is organized as follows. Section~\ref{sec:preliminaries} includes the preliminaries, qualitative and quantitative objectives, RL for continuing tasks, and the problem statements.
Section~\ref{sec:classes_spec} provides our results on qualitative objectives suitable for continual learning.
Section~\ref{sec:result} establishes our theoretical results on average-reward RL for qualitative objectives.
Section~\ref{sec:morl} presents our algorithm for average-reward RL with lexicographic objectives, where the primary goal is to satisfy a qualitative specification and the secondary goal is to optimize a quantitative average-reward objective.
In Section~\ref{sec:case-studies}, we test the performance of our approach on different case studies and compare it against prior techniques. Section~\ref{sec:related-work} discusses related work before concluding in Section~\ref{sec:conclusion}.

%% file: prelims.tex
\subsection{The Environment: Markov Decision Processes}
Let $\DIST(S)$ be the set of distributions over a given set $S$.
\label{subsec:MDP}
\begin{definition} [Markov decision process (MDP)]
An MDP $\Mm$ is a tuple $(S, s_0, A, T, AP, L)$ where
\begin{itemize}
    \item $S$ is a finite set of states, with $s_0 \in S$ as the initial state;
    \item $A$ is a finite set of {\it actions};
    \item $T \colon S \times A \pto \DIST(S)$ is the {probabilistic transition function};
    \item $AP$ is the set of {\it atomic propositions}; and 
    \item $L\colon S \to 2^{AP}$ is the {\it labeling function}.
\end{itemize}
For any state $s \in S$, we let $A(s)$ denote the set of actions that can be selected in
state $s$.
An MDP is a Markov chain if $A(s)$ is a singleton for all $s \in S$.
\end{definition}

\subsubsection*{Runs.}
For states $s, s' \in S$ and $a \in A(s)$, $T(s,a)(s')$ equals the conditional probability $p(s'\!\mid\! s, a)$, which is the probability of jumping to state $s'$ from state $s$ under action $a$.
A \emph{run} of $\Mm$ is an $\omega$-word $\seq{s_0, a_1, s_1, \ldots} \in S
\times (A \times S)^\omega$ such that $p(s_{i+1} \!\mid\! s_{i}, a_{i+1}) > 0$ for all $i
\geq 0$.
A \emph{finite} run is a finite such sequence. 
For a run $r = \seq{s_0, a_1, s_1, \ldots}$ we define the corresponding
\emph{labeled} run as $L(r) = \seq{L(s_0), L(s_1), \ldots} \in (2^{AP})^\omega$.
We write $\Runs^\Mm (\FRuns^\Mm)$  for the set of (finite) runs of the MDP
$\Mm$  and $\Runs{}^\Mm(s) (\FRuns{}^\Mm(s))$  for the set of (finite) runs of
the MDP $\Mm$ from $s$.  We write $\Last(r)$ for the last state
of a finite run $r$. The superscript $\Mm$ will be dropped when clear from context.

\subsubsection*{Strategies or Policies.}
A {\it strategy} in $\Mm$ is a function $\sigma \colon \FRuns \to \DIST(A)$ such that
$\supp(\sigma(r)) \subseteq A(\Last(r))$, where $\supp(d)$ denotes the support
of the distribution $d$.
A memory skeleton is a tuple $M_{\mathfrak s} = (M, \Sigma, m_0, \alpha_{\mathfrak s})$ where $M$ is a finite set of
memory states, $\Sigma$ is a finite alphabet, $m_0\in M$ is the initial state, and $\alpha_{\mathfrak s}: M \times \Sigma \to M$
is the memory update function.
We define the extended memory update function $\hat{\alpha}_{\mathfrak s}: M {\times} \Sigma^* \to M$ inductively 
in the usual way.
A finite-memory strategy for $\Mm$ over a memory skeleton $M_{\mathfrak s}$ is a Mealy machine 
$(M_{\mathfrak s}, \alpha_{\mathfrak a})$ where $\alpha_{\mathfrak a}: S {\times} M \to \DIST(A)$ is the {\it next-action function} 
that suggests the next action based on the MDP and memory state. 
The semantics of a finite-memory strategy $(M, \alpha_{\mathfrak a})$ is given as a strategy 
$\sigma: \FRuns \to \DIST(A)$ such that, for every $r \in \FRuns$, we have that 
$\sigma(r) = \alpha_{\mathfrak a}(\Last(r), \hat{\alpha}_{\mathfrak s}(m_0, L(r)))$.

A strategy $\sigma$ is {\it pure} if $\sigma(r)$ is a point
distribution for  all runs $r \in \FRuns^\Mm$ and a strategy is {\it mixed} (short for
strictly mixed) if $\supp(\sigma(r)) = A(\Last(r))$ for  all runs
$r \in \FRuns^\Mm$.
Let $\Runs^\Mm_\sigma(s)$ denote the subset of runs $\Runs^\Mm(s)$ from initial state $s$ that
follow strategy $\sigma$.
Let $\Strat_\Mm$ be the set of all strategies.
We say that $\sigma$ is \emph{stationary} if $\Last(r) = \Last(r')$ implies
$\sigma(r) = \sigma(r')$ for all runs $r, r' \in \FRuns^\Mm$.
A stationary strategy can be given as a function $\sigma: S \to \DIST(A)$.  
A strategy is \emph{positional} if it is both pure and stationary. 

\subsubsection*{Probability Space.}
An MDP $\Mm$ under strategy $\sigma$ results in a Markov chain $\Mm_\sigma$.
If $\sigma$ is a finite-memory strategy, then $\Mm_\sigma$ is a finite-state
Markov chain.
The behavior of an MDP $\Mm$ under a strategy $\sigma$ and starting state
$s \in S$ is defined on the probability space
$(\Runs^\Mm_\sigma(s), \Ff_{\Runs^\Mm_\sigma(s)}, \Pr^\Mm_\sigma(s))$ over
the set of infinite runs of $\sigma$ with starting state $s$. $\Ff_{\Runs^\Mm_\sigma(s)}$ denotes the sigma-algebra on these runs generated by the underlying MDP, and $\Pr^\Mm_\sigma(s)$ is the probability distribution over these runs constructed inductively. 
Given a random variable 
$f\colon \Runs^\Mm \to \Real$, we denote by $\eE^{\Mm}_{\sigma}(s)\{f\}$
the
expectation of $f$ over the runs of $\Mm$ starting from $s$ that
follow $\sigma$.

\subsubsection*{Structural Properties of MDPs.}
A \emph{sub-MDP} of $\Mm$ is an MDP $\Mm' = (S', A', T', AP, L')$, where $S' \subset
S$, $A' \subseteq A$ is such that $A'(s) \subseteq A(s)$ for all $s \in S'$,
and $T'$ and $L'$ are $T$ and $L$ restricted to $S'$ and
$A'$.
An \emph{end-component}~\cite{Luca98} of an MDP $\Mm$ is a sub-MDP $\Mm'$ that is closed under the transitions in $T'$ and such that for every state pair $s, s' \in S'$ there is a 
strategy that can reach $s'$ from $s$ with positive probability. 
A maximal end-component is an end-component that is maximal under
set-inclusion.
Every state $s$ of an MDP $\Mm$ belongs to at most one maximal end-component.
A {\it bottom strongly connected component} (BSCC) of a Markov chain is any
of its end-components.

 An MDP $\Mm$ is {\it communicating} if it is equal to its (only) maximal end-component.
An MDP $\Mm$ is {\it weakly communicating} if its state space can be decomposed into two sets: in the first set, each state is reachable from every other state in the set under some strategy; in the second set, all states are transient under all strategies, meaning that the probability of starting from $s$ in this set and returning to $s$ is less than one under any strategy.

\subsection{Quantitative Objectives: Probabilistic Reward Machines}
\label{subsec:quantitative}
In the classical RL literature, the learning objective is specified using
\emph{Markovian} reward functions, i.e., a function $\rho: S \times A \times S \to \Real$
assigning utility to state-action pairs. 
A \emph{rewardful} MDP is a tuple $\Mm = (S, s_0, A, T, \rho)$, where $S, s_0, A,$ and $T$ are defined as for MDPs, and $\rho$ is a Markovian reward function.
A rewardful MDP $\Mm$ under a
strategy $\sigma$ determines a sequence of random rewards
${\rho(X_{i-1}, Y_i, X_i)}_{i \geq 1}$, where $X_i$ and $Y_i$ are the random variables denoting the $i$-th state and action, respectively.
  For
$\lambda \in [0, 1)$, the \emph{discounted reward} under  $\sigma$ is defined as:
\[\EDisct(\lambda)^\Mm_\sigma(s) := \lim_{N \to \infty} \eE^\Mm_\sigma(s) \Bigl\{\sum_{i =1}^N\lambda^{i-1} \rho(X_{i-1}, Y_i, X_i)\Bigr\},\]
while the \emph{average  reward} is defined as
\[
\EAvg^\Mm_\sigma(s) := \liminf_{N \to \infty} \frac{1}{N} \eE^\Mm_\sigma(s)\Bigl\{\sum_{i =1}^N \rho(X_{i-1}, Y_i, X_i)\Bigr\}.
\]
The average reward is defined using $\liminf$ of the time-averaged expected reward, which is standard in average-reward maximization~\cite{Puterman94}. For a finite MDP with bounded rewards, optimal stationary policies exist for which the average reward converges. Hence, for optimal policies, $\liminf$, $\limsup$, and the limit coincide and induce the same optimal value.

For an objective $\ECost^\Mm {\in} \{\EDisct(\lambda)^\Mm,
\EAvg^\Mm\}$ and state $s$, we define the optimal reward
$\ECost^\Mm_*(s)$ as $\sup_{\sigma \in \Strat_\Mm} \ECost^\Mm_\sigma(s)$.  
A strategy $\sigma$ is optimal for $\ECost^\Mm$ if
$\ECost^\Mm_\sigma(s) {=} \ECost^\Mm_*(s)$ for all $s \in S$.
Optimal reward and strategies for these objectives can be computed in
polynomial time~\cite{Puterman94}. 

Often, complex learning objectives cannot be expressed using Markovian reward
functions. 
A recent trend is to express learning objectives using finite-state reward machines (see, e.g., \cite{icarte2018using, icarte2022reward}), which provides a structured formalism to decompose high-level tasks and exploits the internal structure of the reward function. 
These reward machines often serve as an intermediate representations for formal languages, enabling the automatic translation of logical specifications into reward structures~\cite{camacho2019ltl}. 
Recent works also include learning reward machines directly from agent interactions to handle non-Markovian rewards~\cite{gaon2020reinforcement} and partial observability~\cite{Icarte2023}, with extensions to multi-agent stochastic games~\cite{hu2024reinforcement} and inverse reinforcement learning~\cite{shehab2025learning}.

For the objectives we consider, we require a more expressive variant of reward machines that allows probabilistic and nondeterministic transitions, as well as spurious $\epsilon$-labeled transitions. We refer to these as probabilistic reward machines, with the understanding that any of the three transition types may be absent. For instance, the definition reduces to that of a nondeterministic reward machine as in \cite{kazemi2022translating} when all enabled transitions are taken with probability one. The use of $\epsilon$-transitions serves only to streamline the presentation of our technical results.

\begin{definition}[Probabilistic Reward Machines]
A \emph{probabilistic reward machine} is a tuple $\Rr = (\Sigma_\epsilon, U\times U^p, (u_0, u^p_0), \delta_r, T^p, \rho)$
where
\begin{itemize}
    \item $\Sigma_\epsilon = (\Sigma\cup \{\epsilon\})$ with $\Sigma$ being a finite alphabet and $\epsilon$ indicating a silent transition,
    \item 
    $U$ and $U^p$ are two finite sets of states, 
    \item $(u_0, u^p_0) \in U\times U^p$ is the starting state,
    \item $\delta_r: U \times \Sigma_\epsilon \to 2^U$ is the transition relation (that allows nondeterminism), 
    \item $T^p \colon U^p \pto \DIST(U^p)$ is the probabilistic transition function, and 
\item $\rho: U \times U^p \times \Sigma_\epsilon \times U \times U^p \to \Real$ is the reward function. 
\end{itemize}
    To simplify the notation we let $U^{\Rr}= U \times U^p$ and $u^{\Rr}_0= (u_0, u^p_0)$.
\end{definition}

\begin{definition}[Product (rewardful) MDP]
Given an MDP $\Mm = (S, s_0, A, T, AP, L)$, a reward machine $\Rr = (\Sigma_\epsilon, U^\Rr, u^\Rr_0, \delta_r, T^p, \rho)$ with the alphabet $\Sigma = 2^{AP}$, and the labeling function $L:S\rightarrow\Sigma$, their product
\[\Mm\times\Rr = (S{\times} U^\Rr, s_0 {\times} u^\Rr_0, (A {\times} U^\Rr) \cup\{\epsilon\},
T^\times, \rho^\times)
\]
is a rewardful MDP where
$T^\times{:} (S {\times} U^\Rr) \times ((A {\times} U^\Rr) \cup \{\epsilon\}) \to \DIST(S{\times} U^\Rr)$ is such that
\begin{multline*}
T^\times((s,(u, u^p)), \alpha)(({s}',({u}', {u^p}'))) =
\begin{cases}
T(s,a)({s}')\cdot T^p(u^p)({u^p}') & \text{if } \alpha = (a, ({u}',{u^p}')) \text{ and } (u,L(s),{u}') \in \delta_r \\
1 & \text{if } \alpha = \epsilon, s = s', \text{ and } \delta(u, \epsilon, {u}') \in \delta_r \\
0 & \text{otherwise.}
\end{cases}
\end{multline*}
and $\rho^\times: (S{\times} U^\Rr) \times ((A {\times}
U^\Rr) \cup \{\epsilon\}) \times (S{\times} U^\Rr)\to \Real$ is defined such that 
\begin{equation*}
\rho^\times((s,u^\Rr), \alpha, (s', {u^\Rr}'))=
\begin{cases}
\rho(u^\Rr, L(s), {u^\Rr}') & \text{if } \alpha = (a, {u^\Rr}')\\ %
\rho(u^\Rr, \epsilon, {u^\Rr}') & \text{if } \alpha = \epsilon.
\end{cases}
\end{equation*}
\end{definition}
The Product MDP effectively runs the MDP and the reward machine in parallel. When an action $a\in A$ is taken, a transition occurs in the MDP to update its state, and an appropriate transition occurs in the reward machine to update its state. When an $\epsilon$-action is taken, the state of the MDP remains the same, and one of the $\epsilon$-transitions of the reward machine is activated to update its state.

A deterministic reward machine is retrieved from the above definition by setting
$\mid\!\delta_r(u,a)\!\mid\le 1$ and
$T^p(u^p)\in\{0,1\}$ for all $u\in U$, $a\in\Sigma_\epsilon$, and $u^p\in U^p$.
For technical convenience, we assume that $\Mm{\times}\Rr$ contains only states reachable from $(s_0, u_0^\Rr)$.
For both discounted and average objectives, the optimal strategies of
$\Mm{\times}\Rr$ are positional on $\Mm{\times}\Rr$.
Moreover, these positional strategies characterize a finite memory strategy (with memory skeleton based on the  states of $\Rr$ and the next-action function based on the positional strategy) over $\Mm$
maximizing the learning objective given by $\Rr$. 

\subsection{Qualitative Objectives: Omega-Regular Specifications}
\label{subsec:qualitative}
Formal specification languages, such as B\"uchi automata and logic, provide a rigorous and unambiguous mechanism to express learning
objectives for continuing tasks.
There is a growing
trend~\cite{hahn2019omega,sadigh2014learning,hasanbeig2019certified,bozkurt2019control,kazemi2020formal}
toward expressing learning objectives in RL using linear temporal logic (LTL) and
$\omega$-regular languages (which strictly generalize LTL).
We will describe $\omega$-regular languages by \emph{good-for-MDP} B\"uchi automata~\cite{Hahn20}.

Linear Temporal Logic (LTL) \cite{BK08} is a temporal logic that is often used to  specify objectives in human-readable form.  
Given a set of atomic propositions $AP$, the LTL formulae over $AP$ can be defined via the following grammar: 
\begin{equation}
\varphi := a \mid \neg \varphi \mid \varphi \vee \varphi \mid  \nextt \varphi \mid \varphi \until  \varphi,
\end{equation}
where $a\in AP$, using negation $\neg$, disjunction $\vee$, next $\nextt$, and until $\until$ operators.
Additional operators are introduced as abbreviations:
$\top \rmdef a \vee \neg a$; $\bot \rmdef \neg \top$;
$\varphi \wedge \psi \rmdef \neg (\neg \varphi \vee \neg \psi)$;
$\varphi \rightarrow \psi \rmdef \neg \varphi \vee \psi$;
$\eventually \varphi \rmdef \top \until \varphi$; and
$\always \varphi \rmdef \neg \eventually \neg \varphi$.  We write $w
\models \varphi$ if
$\omega$-word $w$ over $2^{AP}$ satisfies LTL formula $\varphi$. 
The satisfaction relation is defined inductively
\cite{BK08}.
We will provide details of classes of specifications including safety, liveness, and fairness in Subsection~\ref{sec:classes_spec}.

Nondeterministic B\"uchi automata are finite state acceptors for
all $\omega$-regular languages.
\begin{definition}[B\"uchi Automata]
A (nondeterministic) \emph{B\"uchi automaton} is a tuple
${\mathcal A} = (\Sigma,Q,q_0,\delta,F)$, where 
\begin{itemize}
    \item $\Sigma$ is a finite \emph{alphabet}, 
    \item $Q$ is a finite set of \emph{states}, 
    \item $q_0 \in Q$ is the \emph{initial state}, 
\item $\delta \colon Q \times \Sigma \to 2^Q$ is the \emph{transition function}, and 
\item $F \subset Q \times \Sigma\times Q$ is the set of \emph{accepting transitions}.    
\end{itemize}

A \emph{run} $r$ of ${\mathcal A}$ on $w \in \Sigma^\omega$
is an $\omega$-word $\seq{r_0, w_0, r_1, w_1, \ldots}\in(Q \times \Sigma)^\omega$ such that $r_0 = q_0$ and, for $i > 0$,
$r_i \in \delta(r_{i-1},w_{i-1})$.  Each triple
$(r_{i-1},w_{i-1},r_i)$ is a \emph{transition} of ${\mathcal A}$.
We write $\infi(r)$ for the set of transitions that appear infinitely
often in the run $r$.
A run $r$ of ${\mathcal A}$ is \emph{accepting} if $\infi(r) \cap
F \neq \emptyset$. 
The \emph{language} $\Ll(\Aa)$ of ${\mathcal A}$ is the subset of words in
$\Sigma^\omega$ with accepting runs in ${\mathcal A}$.
A language is $\omega$-\emph{regular} if it is accepted
by a B\"uchi automaton.
\end{definition}

\subsubsection*{Good-for-MDP B\"uchi Automata.}
Given an MDP $\Mm$ and a specification $\varphi$ represented with
a B\"uchi automaton $\Aa = (\Sigma,Q,q_0,\delta,F)$,
we want to compute an optimal strategy satisfying the
objective.
We define the satisfaction probability of $\sigma$ from starting state
$s$ as 
\begin{equation*}
\PSemSat^\Mm_{\Aa}(s, \sigma)
=   \Pr{}^\Mm_\sigma(s) \bigl\{ r \in \Runs_\sigma^\Mm(s) \colon
  L(r) \in \Ll(\Aa) \bigr\}.
\end{equation*}
The optimal satisfaction probability
$\PSemSat^\Mm_{\Aa}(s)$ for specification $\Aa$
is defined as $\sup_{\sigma \in \Strat_\Mm} \Pr^\Mm_\sigma(s, \sigma)$ and we say
that $\sigma \in \Strat_\Mm$ is an optimal strategy for $\Aa$ if
$\PSemSat^\Mm_{\Aa}(s, \sigma) (s) = \PSemSat^\Mm_{\Aa}(s)$.

\begin{definition}[Product MDP]
Given an MDP
$\Mm = ( S, s_0, A, T, AP, L )$
and automaton
$\mathcal{A} = (\Sigma, Q, q_0, \delta, F )$ with alphabet $\Sigma = 2^{AP}$,
the
\emph{product}
$\Mm \times \mathcal{A} = ( S \times Q, (s_0,q_0), A \times Q, T^\times, F^\times )$
is an MDP with initial
state $(s_0,q_0)$
and accepting transitions $F^\times$ where
$T^\times : (S \times Q) \times (A \times Q) \pto \DIST(S \times Q)$
is defined by
\begin{equation*}
T^\times((s,q),(a,q'))(({s}',{q}')) =
\begin{cases}
T(s,a)({s}') & \text{if } (q,L(s),{q}') \in \delta \\
0 & \text{otherwise.}
\end{cases}
\end{equation*}
The accepting transition set
$F^\times \subseteq (S \times Q) \times (A \times Q) \times (S
\times Q)$ is defined by $((s,q),(a,q'),(s',q')) \in F^\times$
if, and only if, $(q,L(s),q') \in F$ and $T(s,a)(s') > 0$.
A strategy $\sigma^\times$ on the product defines a strategy $\sigma$ on the
MDP with the same value, and vice versa.  Note that for a stationary
$\sigma^\times$, the strategy $\sigma$ may need memory.
End-components and runs of the product MDP are defined just like for MDPs.    
\end{definition}

A run $r$ of $\Mm {\times} \mathcal{A}$ is accepting if
$\infi(r) \cap F^\times \neq \emptyset$.
The \emph{syntactic satisfaction}
probability is the probability of accepting runs:
\[
\PSat^\Mm_{\Aa}((s,q), \sigma^\times) = 
\Pr{}^{\Mm\times\Aa}_{\sigma^\times}(s,q) \Set{ r \in
\Runs_{\sigma^\times}^{\Mm\times\Aa}(s,q) : \infi(r) \cap F^\times
\neq \emptyset }.
\]
Similarly, we define $\PSat^\Mm_{\Aa}(s)$ as the optimal probability over the
product, i.e., $\sup_{\sigma^\times}\big(\PSat^\Mm_{\Aa}((s,q_0),
\sigma^\times)\big)$.
For a deterministic $\Aa$ the equality $\PSat^\Mm_{\Aa}(s) = \PSemSat^\Mm_{\Aa}(s)$ holds; however, this is not guaranteed for nondeterministic
B\"uchi automata as the optimal resolution of nondeterministic
choices may require access to future events.
This motivates the definition of a good-for-MDP nondeterminisitc B\"uchi automata.
\begin{definition}[Good-for-MDP~\cite{Hahn20}]
A B\"uchi automaton $\Aa$ is \emph{good for MDPs} (GFM),
if $\PSat^\Mm_{\Aa}(s_0) = \PSemSat^\Mm_{\Aa}(s_0)$ holds for all
MDPs $\Mm$ and starting states $s_0$.
\end{definition}
Note that every $\omega$-regular objective can be expressed as a GFM automaton \cite{Hahn20}.
A popular class of GFM automata is suitable limit-deterministic B\"uchi automata~\cite{Hahn2015LazyPM,sickert2016limit}.
There are other types of automata that are good-for-MDPs.  For example, in parity automata, each transition is assigned an integer priority.  A run of a parity automaton is accepting if the highest recurring priority is odd.  In this paper, we only use GFM B\"uchi automata.

The satisfaction of an $\omega$-regular objective given as a GFM automaton $\Aa$
by an MDP $\Mm$ can be formulated in terms of the accepting maximal
end-components of the product $\Mm {\times} \mathcal{A}$, i.e., the maximal
end-components that contain an accepting transition from $\F^\times$.  
The optimal satisfaction probabilities and strategies can be computed by
first computing the accepting maximal end-components of $\Mm\times \Aa$ and then
maximizing the probability to reach states in such components.
The optimal strategies are positional on $\Mm\times\Aa$ and induce finite-memory strategies over $\Mm$ that maximize the satisfaction probability of the learning objective given by $\Aa$.

\subsection{Reinforcement Learning for Continuing Tasks and Problem Statements}
\label{subsec:RL_continual}

RL is a sampling-based optimization approach where an agent learns to optimize its strategy by repeatedly interacting with the environment relying on the reinforcements (numerical reward signals) it
receives for its actions.
We focus on the model-free approach to RL, where the learner computes optimal strategies without explicitly estimating the transition probabilities and rewards.
These approaches are asymptotically more space-efficient \cite{Strehl06} than model-based RL and have been shown to scale well \cite{Mnih15,AlphaGo}.
Some prominent model-free RL algorithms for discounted and average reward objectives include Q-learning and TD($\lambda$)~\cite{Sutton18} and Differential Q-learning~\cite{wan2021learning}.

In some applications, such as running a maze or playing tic-tac-toe, the
interaction between the agent and the environment naturally breaks into finite
length learning sequences, called episodes.
Thus the agent optimizes its strategy by combining its experience over different
episodes.
We call such tasks \emph{episodic}.
On the other hand, for some applications---such as process control and reactive
systems---this interaction continues ad-infinitum and the agent learns
over a single lifetime.
We call such tasks \emph{continuing}.
This paper develops a model-free RL algorithm for continuing tasks where the learning objective is an $\omega$-regular specification given as a GFM automaton.
Prior solutions~\cite{hahn2019omega,sadigh2014learning,hasanbeig2019certified,bozkurt2019control}
focused on episodic setting and have proposed a model-free reduction from
$\omega$-regular objectives to discounted-reward objectives.
Several researchers~\cite{Naik2019DiscountedRL,Sutton18} have made the case
for adopting average reward formulation for continuing tasks due to several
limitations of discounted-reward RL in continuing tasks.
This paper investigates a model-free reduction from $\omega$-regular objectives
to average-reward objectives in model-free RL.

\begin{problem}[$\omega$-Regular to Average Reward Translation]
\label{prob}
  Given an unknown communicating MDP $\Mm = (S, s_0, A, T, AP, L)$ and a GFM automaton $\Aa =
  (\Sigma, Q, q_0,\delta, F)$, the \emph{reward translation} 
  problem is to design a reward machine $\Rr$ such that an optimal positional strategy maximizing the average reward for $\Mm{\times}\Rr$ provides a finite memory strategy maximizing
  the satisfaction probability of $\Ll(\Aa)$ in $\Mm$. 
\end{problem}

We also consider the lexicographic multi-objective optimization in
which the goal is to maximize a mean-payoff objective over those policies that satisfy a given liveness specification.

\begin{problem}[Lexicographic $\omega$-regular and Average Reward]
\label{prob_lexo}
  Given an unknown communicating MDP $\Mm = (S, s_0, A, T, AP, L)$, a GFM automaton $\Aa =
  (\Sigma, Q, q_0,\delta, F)$, and a reward function $\rho(s, s')$, the problem is to design an algorithm such that the resulting policy maximizes the average reward
  \[
\EAvg^\Mm_\sigma(s) := \liminf_{N \to \infty} \frac{1}{N} \eE^\Mm_\sigma(s)\left\{\sum_{i=1}^N \rho(s_{i-1}, a_i, s_i)\right\},
\]
  among the policies that maximize the satisfaction probability of $\Aa$ in $\Mm$. 
\end{problem}

Next we give an account of challenges with average-reward RL, identify suitable classes of qualitative temporal specifications in Section~\ref{sec:classes_spec}, and then solve Problems~\ref{prob}-\ref{prob_lexo} respectively in Sections~\ref{sec:result}-\ref{sec:morl} under appropriate assumptions on the structure of the underlying MDP.

\subsection{Average-Reward Reinforcement Learning}
\label{subsec:challenges}

Given an MDP $\Mm$, reward machine $\Rr$, and an objective (discounted or average reward), 
an optimal strategy can be computed in polynomial time using linear programming~\cite{Puterman94}.  
Similarly, graph-theoretic techniques to find maximal end-components can be
combined with linear programming to compute optimal strategies for
$\omega$-regular objectives~\cite{Hahn2015LazyPM}. 
However, such
techniques are not applicable when the transitions or reward structure of the MDP is unknown.
The existing average-reward RL algorithms such as Differential Q-learning provide
convergence guarantees under the assumption that the MDP $\Mm$ is 
communicating~\cite{wan2021learning}.
Thus, for the reward translation in Problem~\ref{prob} to be effective, we need the product $\Mm\times\Aa$ to be communicating. 
Unfortunately, even when $\Mm$ is communicating,
$\Mm\times\Aa$ may violate the communicating requirement. We tackle this issue in Section~\ref{sec:result}. 

The communicating assumption requires that, for every pair of states $s$ and $s'$, there is a policy under which the MDP can reach $s'$ from $s$ in a finite number of steps with positive probability.
Under the communicating assumption, there exists a unique
optimal reward rate $r^\star$ that is independent of the start state.
The differential Q-learning algorithm constructs a table of estimates $Q_t:S\times U\rightarrow\mathbb R$ at each time step $t$. Let us denote the state-action pair of the MDP at time $t$ by $(s_t,u_t)$.
Then, differential Q-learning updates the table as follows:
\begin{align*}
Q_{t+1}(s,u) &:= Q_t(s,u)&\text{ for all } (s,u) &\neq (s_t, u_t)\\
      Q_{t+1}(s,u) &:= Q_t(s_t,u_t) + \alpha_t\delta_t, &\text{ for } (s,u) &= (s_t, u_t),
\end{align*}
where $\alpha_t$ is the step-size at step $t$, and $\delta_t$ is the temporal difference error defined as
\[
\delta_t := \rho_{t+1} - \bar r_t + \max _{u}Q_t(s_{t+1},u) - Q_t(s_{t},u_{t}) \text{ and } 
    \bar r_{t+1} := \bar r_t+\eta\cdot\alpha_t\cdot\delta_t,
\]
where $\rho_{t+1}$ is the reward received for the transition $(s_t,u_t,s_{t+1})$,
$\bar r_t$ is a scalar estimate of the optimal reward rate $r^\star$,
and $\eta$ is a positive constant.
It is shown in~\cite{wan2021learning} that $\bar{r}_t$ converges to $r^\star$ under the following assumptions:
\begin{enumerate}[(a)]
\item the MDP is communicating;
\item the associated Bellman equation has a unique solution up to a constant;
\item the step sizes $\alpha_t$ decrease appropriately with $t$;
\item all state–action pairs are updated infinitely often; and
\item the ratio of the update frequency of the most-updated state–action pair to the least-updated one is finite.
\end{enumerate}

\section{Specifications Suitable for Continual Learning}
\label{sec:classes_spec}

In this section, we study and characterize temporal logic specifications that express various classes of tasks. We then identify those suitable for continual reinforcement learning and provide solutions to Problems~\ref{prob}–\ref{prob_lexo} in subsequent sections.
Two important classes of temporal logic specifications are \emph{safety} and \emph{liveness} properties~\cite{alpern1985defining,manna1990hierarchy}. Loosely speaking, safety specifications ensure that \emph{something~bad ~never~happens}, while liveness specifications guarantee that \emph{something good eventually happens}. 
Below, we briefly recall the formal definition of safety and then focus on liveness.

\begin{definition}[Safety and Liveness]
Let $w$ be a sequence over alphabet $\Sigma$ and let $w(i)$ represent the prefix of $w$ of length $i+1$; i.e., $x(i)= (x_0, x_1, \ldots, x_i)$. 
\begin{itemize}
    \item A language $\Ll\subseteq \Sigma^{\omega}$ is a safety specification if, and only if, for all $\omega$-sequences $w\in \Sigma^{\omega}$ not in $\Ll$,
some prefix of $w$ cannot be extended to an $\omega$-sequence in the language. 
\item A language $\Ll \subseteq \Sigma^\omega$ is a liveness specification if, and only if, we can extend any finite sequence in $\Sigma^*$ to an $\omega$-sequence in the language. In other words, the set $\{w(i): w\in \Ll\}$ is $\Sigma^*$. For instance $\psi_1 = a \vee \eventually b$ is a liveness specification.
\end{itemize}
\end{definition}

\begin{definition}[Absolute Liveness]
 A language $\Ll \subseteq \Sigma^\omega$ is an \emph{absolute liveness} specification if, and only if, $\Ll$ is non-empty and, for every finite sequence $\bar w\in \Sigma^*$ and every $w\in \Ll$, $\bar w w\in \Ll$ holds.  That is, appending an arbitrary finite
prefix to an accepted word produces an accepted word. This implies that an LTL specification $\varphi$ is absolute liveness specification if $\varphi$ is satisfiable and $\varphi$ is equivalent to $\eventually \varphi$.  One can observe that $\psi_1$ is not an absolute liveness specification since adding the prefix $\neg a$ to a trace that does not satisfy $\eventually b$ yields a trace not in the language. However, $\psi_2= \eventually a$ is an absolute liveness specification.     
\end{definition}

\begin{definition}[Stable Specification]
A language $\Ll \subseteq \Sigma^\omega$ is a \emph{stable} specification if, and only if, it is closed under suffixes; that is, if $x$ is in $\Ll$, then every suffix of $x$ is also in $\Ll$. Since deleting the first element from $a(\neg a)^\omega$ results in a sequence that does not satisfy $\psi_2$, $\psi_2$ is not stable. However, $\psi_3= \always a \vee \always\eventually b$ is stable. Moreover, a language is a \emph{fairness} specification if and only if it is both a stable specification and absolute liveness specification, e.g., $\psi_4= \always\eventually a$. One can easily conclude that none of the specifications $\psi_1, \psi_2$, and $\psi_3$ are fairness specifications.       
\end{definition}

We give a solution for the translation in Problem~\ref{prob} for \emph{absolute liveness} specifications~\cite{sistla1994safety}.
Adding a prefix to a trace for average reward objectives should not change the average value associated with the trace. This is aligned with the satisfaction of absolute liveness specifications. Moreover, since absolute liveness specifications cannot be rejected for any finite word, they preserve the continual nature of the learning procedure. In the following, we focus on the characteristics of absolute liveness specifications.

The set of absolute liveness specifications is closed under union and intersection.
The proofs mentioned below rely on the property $\varphi \equiv \eventually\varphi$ in LTL notation.  The more general $\omega$-regular proofs use $\varphi \equiv \Sigma^*\varphi$. In addition, closure under intersection requires an ``asterisk:'' the intersection is an absolute liveness specification unless it is empty.
Suppose $\varphi_1$ is equivalent to $\eventually \varphi_1$ and $\varphi_2$ is equivalent to $\eventually \varphi_2$.  Then,
\begin{equation*}
    \varphi_1 \vee \varphi_2 \equiv (\eventually \varphi_1) \vee (\eventually \varphi_2) \equiv \eventually(\varphi_1 \vee \varphi_2) \enspace.
\end{equation*}
Moreover, since for every LTL specification $\varphi$, $\eventually\varphi$ is equivalent to $\eventually\eventually \varphi$,
\begin{equation*}
    \varphi_1 \wedge \varphi_2 \equiv (\eventually\eventually \varphi_1) \wedge (\eventually\eventually \varphi_2) \equiv \eventually((\eventually\varphi_1) \wedge (\eventually\varphi_2)) \equiv \eventually(\varphi_1 \wedge \varphi_2) \enspace.
\end{equation*}

Lemma~2.1 of \cite{sistla1994safety}, which says that the complement of a stable specification different from $\Sigma^\omega$ is an absolute liveness specification and vice versa, can be used to show that the stable specifications are also closed under union and intersection.
The case when either stable specification is $\top$ is easily proved.
Because of these closure properties, $\bigvee_{1 \leq k \leq n}(\eventually\always a_k) \wedge (\always\eventually b_k)$ (where each $a_k$ and $b_k$ is an atomic proposition) is an absolute liveness specification of Rabin index $n$.  This means that there is no upper bound to the number of priorities required by a parity automaton accepting an absolute liveness specification \cite{Wagner79}. The simplest absolute liveness specification is $\top$, which requires one priority.
The automata that accept absolute liveness and stable specifications have special properties discussed next.

\begin{lemma}
\label{th:al-containment}
Let $\mathcal{A}$ be a deterministic Büchi automaton with initial state $s_0$. The automaton $\mathcal{A}$ accepts an absolute liveness specification if and only if, for every reachable state $s$ of $\mathcal{A}$, the language accepted from $s$ contains the language accepted from $s_0$.
\end{lemma}
\begin{proof}
Let $u$ be a word that takes the automaton $\mathcal{A}$ from the initial state $s_0$ to a reachable state $s$, and let $w$ be a word accepted from $s_0$.
Suppose first that $\mathcal{A}$ accepts an absolute liveness specification. Since $w$ is accepted from $s_0$, the concatenated word $uw$ is also accepted from $s_0$. Because $\mathcal{A}$ is deterministic, the unique run of $uw$ starting from $s_0$ reaches $s$ after reading $u$ and then continues with $w$. It follows that $w$ is accepted from $s$.
Conversely, suppose that for every reachable state $s$ of $\mathcal{A}$, the language accepted from $s$ contains the language accepted from $s_0$. Let $u$ be a word that takes $\mathcal{A}$ from $s_0$ to $s$, and let $w$ be a word accepted from $s_0$. By assumption, $w$ is accepted from $s$, which implies that the concatenated word $uw$ is accepted from $s_0$. Therefore, $\mathcal{A}$ accepts an absolute liveness specification.
\end{proof}
The ``only-if'' direction of Lemma~\ref{th:al-containment} does not extend to nondeterministic automata. This failure is witnessed by the two-state automaton shown in Fig.~\ref{fig:example_multichain}(a), which recognizes the absolute liveness specification $\eventually\always a$. In this automaton, the unique accepting state recognizes precisely the language $\always a$. An analogous separation result holds for stable specifications.

\begin{lemma}
\label{th:st-containment}
Let $\mathcal{A}$ be a B\"uchi automaton with initial state $s_0$.
The automaton $\mathcal{A}$ accepts a stable specification if and only if,
for every reachable state $s$ of $\mathcal{A}$, the language accepted from $s$
is contained in the language accepted from $s_0$.
\end{lemma}
\begin{proof}
Let $u$ be a word that takes the automaton $\mathcal{A}$ from the initial state
$s_0$ to a reachable state $s$, and let $w$ be a word accepted from $s$.
Suppose first that $\mathcal{A}$ accepts a stable specification.
Then the concatenated word $uw$ is accepted from $s_0$.
Since stable specifications are closed under taking suffixes, it follows that
the suffix $w$ of $uw$ is also accepted from $s_0$.
Conversely, suppose that for every reachable state $s$ of $\mathcal{A}$,
the language accepted from $s$ is contained in the language accepted from $s_0$.
Let $u$ take $\mathcal{A}$ from $s_0$ to $s$ along an accepting run of $uw$.
Since $w$ is accepted from $s$, the assumption implies that $w$ is accepted
from $s_0$. Hence, $uw$ is accepted from $s_0$, and therefore
$\mathcal{A}$ accepts a stable specification.
\end{proof}
Note that Lemma~\ref{th:st-containment} does not require determinism.
Combined with Lemma~\ref{th:al-containment}, it proves that a deterministic automaton accepts a fairness specification if, and only if, all its states are language-equivalent.  This implies that a fairness specification accepted by a deterministic (B\"uchi) automaton is accepted by a strongly connected deterministic (B\"uchi) automaton.  Any reachable sink SCC of a deterministic automaton for a fairness specification is, in itself, an automaton for the specification.
The above lemmas show that checking whether a specification is stable or an absolute liveness specification is reducible to checking language containment for deterministic B\"uchi automata.
To solve Problem~\ref{prob}, we make the following assumption.

\begin{assumption}
 Given an MDP $\Mm$ and B\"uchi automaton $\Aa$, we assume that:
1) $\Mm$ is communicating; 
2) $\Aa$ is a GFM automaton; and
3) $\Aa$ accepts an absolute liveness specification.
\end{assumption}

We also study relaxing the communicating assumption on $\Mm$ to weakly communicating, which requires strengthening the specification to fairness in order to retain correctness and convergence guarantees.

%% file: methodology.tex
This section provides a solution for Problem~\ref{prob}.
Let us fix a communicating MDP $\Mm = (S, s_0, A, T, AP, L)$ and an absolute
liveness GFM property $\Aa = (\Sigma, Q, q_0, \delta, F)$ for the rest of this section.  
Our goal is to construct a reward machine $\Rr$ such that we can use an off-the-shelf average reward RL on $\Mm\times\Rr$ to compute an optimal strategy of $\Mm$ against $\Aa$.
Since the optimal strategies are not positional on $\Mm$ but rather positional
on $\Mm{\times}\Aa$, it is natural to assume that the reward machine $\Rr$ takes the structure of $\Aa$ with a reward function providing positive reinforcement with every accepting transition. 
Unfortunately, even for absolute liveness GFM automata $\Aa$, the product
$\Mm\times\Aa$ with a communicating MDP $\Mm$ may not be
communicating.

\begin{figure}[t]
    \centering
    \begin{minipage}[c]{0.48\textwidth}
        \centering
        \begin{tikzpicture}[
            >=latex,
            node distance=15mm,
            auto,
            every state/.style={fill=yellow!20}
        ]
            \node[state, initial, initial text={}, initial where=left] (q1) {$\eventually\always a$};
            \node[state, accepting, right=of q1] (q2) {$\always a$};

            \path[->]
                (q1) edge node[above] {$a$} (q2)
                (q1) edge[loop above] node {$\top$} (q1)
                (q2) edge[loop above] node {$a$} (q2);
        \end{tikzpicture}
        \caption*{(a) A nondeterministic B\"uchi automaton for the absolute liveness specification $\eventually\always a$, used as a counterexample to the ``only-if'' direction of Lemma~\ref{th:al-containment}.}
    \end{minipage} 
    \hfill
    \begin{minipage}[c]{0.48\textwidth}
        \centering
        \begin{tikzpicture}[
            >=latex,
            node distance=10mm,
            auto,
            every state/.style={fill=yellow!20}
        ]
            \node[state, initial, initial text={}, initial where=below] (q0) {$q_0$};
            \node[state, accepting, left=of q0] (q1) {$q_1$};
            \node[state, accepting, right=of q0] (q2) {$q_2$};

            \path[->]
                (q0) edge node[above] {$a$} (q1)
                (q0) edge node {$\neg a$} (q2)
                (q0) edge[loop above] node {$\top$} (q0)
                (q1) edge[loop above] node {$a$} (q1)
                (q2) edge[loop above] node {$\neg a$} (q2);
        \end{tikzpicture}
        \caption*{(b) Automaton for the absolute liveness specification $(\eventually\always a)\,\vee\,(\eventually\always \neg a)$ used in Example~\ref{example:noncommunication}.}
    \end{minipage}
    \caption{Nondeterministic B\"uchi automata illustrating absolute liveness specifications.}
    \label{fig:example_multichain}
\end{figure}

\begin{example}
\label{example:noncommunication}
Assume a communicating MDP $\Mm$ with at least one state labeled $a$ or $\neg a$, and the absolute liveness property $\varphi = (\eventually\always a) \vee (\eventually\always \neg a)$ with its automaton shown in
Fig.~\ref{fig:example_multichain}(b). Observe that any run that visits one of the two accepting states cannot visit the other one.
Hence, the product does not satisfy the communicating property.
\end{example}

\subsubsection*{Reward Machine Construction.}
Let $\Aa = (\Sigma, Q, q_0, \delta, F)$ be an absolute liveness GFM
B\"uchi automaton.
Consider $\Rr_\Aa = (\Sigma_\epsilon, Q, q_0, \delta', \rho)$ where $\delta'(q, a) = \delta(q, a)$ for all $a \in \Sigma$ and $\epsilon$ transitions reset to the starting state, i.e. $\delta'(q, \epsilon) = q_0$.
Note that by adding the reset ($\epsilon$) action from every state of $\Rr$ to its initial state, the graph structure of $\Mm$ is strongly connected.
The reward function $\rho: Q \times (\Sigma\cup\{\epsilon\}) \times Q{\to} \Real$ is such that 
\begin{equation*}
\rho(q, a, q')  =
\begin{cases}
        c & \text{ if $a  = \epsilon$}\\
        1 & \text{ if $(q, a, q') \in F$} \\
        0 & \text{otherwise.}
\end{cases} 
\end{equation*}

\begin{remark}
The above construction that adds resets transitions to the starting state could also be seen from the lens of constructing a GFM parity automaton.
This would be performed as follows. First, transform the original GFM B\"uchi automaton into a nondeterministic parity automaton by assigning parity equal to two for accepting transitions and equal to one for other transitions. Then, add reset transitions with parity equal to three.
Therefore, 
reset transitions are rejecting at the highest priority. This parity automaton preserves the language of the original B\"uchi automaton: since an accepting run involves only finitely many resets, any accepted word effectively consists of a finite prefix followed by a suffix accepted by the original automaton. Because absolute liveness properties are closed under prefix addition, the resulting word remains in the language. Furthermore, adding transitions yields an automaton that simulates the original one. Consequently, by invoking \cite[Lemma~1]{Hahn20}, we conclude that the resulting parity automaton remains GFM.
\end{remark}

\begin{lemma}[Preservation of Communication]
\label{l1}
For a communicating MDP $\Mm$ and reward machine $\Rr_\Aa$ for
an absolute liveness GFM automaton $\Aa$, we have that the product $\Mm{\times}\Rr_\Aa$ is
communicating. 
\end{lemma}
\begin{proof}
To show that $\Mm{\times}\Rr_\Aa$ is communicating, we need to show that for arbitrary states $(s, q), (s', q') \in S\times Q$ reachable from the initial state $(s_0, q_0)$, there is a strategy that can reach $(s', q')$ from $(s, q)$ with positive probability. 
Note that since $\Mm$ is communicating, it is possible to reach $(s_0, q'')$ from $(s, q)$ for some $q''$ of $\Rr_\Aa$ using a strategy to reach $s_0$ from $s$ in $\Mm$. We can then use a reset ($\epsilon$) action in $\Rr_\Aa$ to reach the state $(s_0, q_0)$. Since $(s', q')$ is reachable from the initial state $(s_0, q_0)$, we have a strategy to reach $(s', q')$ from $(s, q)$ with positive probability.
\end{proof}

\begin{lemma}[Average and Probability]\label{mm}
There exists a constant $c^* < 0$ such that for all $c < c^*$, positional
strategies that maximize the average reward on $\Mm\times \Rr_\Aa$ will maximize the satisfaction probability of $\Aa$.
\end{lemma}
\begin{proof}
The proof is in three parts.
\begin{enumerate}
    \item 
First observe that if $c < 0$, then for any average-reward optimal strategy in $\Mm\times \Rr_\Aa$, the expected average reward is non-negative. 
This is so because all other actions except $\epsilon$ actions provide non-negative rewards. Hence, any strategy that takes $\epsilon$ actions only finitely often, results in a non-negative average reward.

\item Let $\Strat^*$ be the set of positional strategies in $\Mm{\times}\Rr_\Aa$ such that the $\epsilon$ actions are taken only finitely often, i.e. no BSCC of the corresponding Markov chain contains an $\epsilon$ transition. 
Let $\Strat^\epsilon$ be the set of remaining positional strategies, i.e., the set of positional strategies that visit an $\epsilon$ transition infinitely often.
Let $0{<}p_{\tt min}{<}1$ be a lower bound on the expected long-run frequency of the $\epsilon$ transitions among all strategies in $\Strat^\epsilon$.
Let $c^* = -1/p_{\tt min}$. 
Observe that for every policy $\sigma' \in \Strat^\epsilon$, the expected average reward is negative and cannot be an optimal strategy in $\Mm\times\Rr_\Aa$.
To see that, let $0 < p \le 1$ be the the long-run frequency of the $\epsilon$ transitions for $\sigma$ and let $0 \leq q < 1$ be the long-run frequency of visiting accepting transitions for $\sigma$. The average reward for $\sigma$ is
\begin{eqnarray*}
\EAvg^{\Mm\times\Rr_\Aa}_\sigma(s_0,q_0) &=& p \cdot c + q \cdot 1 + (1-p-q) \cdot 0 \\
&\leq&  p \cdot c + q \cdot 1 + (1-p-q) \cdot 1  =  p \cdot c + (1-p) \\
&\leq&  p \cdot c^* + (1-p) =  -p/ p_{\tt min} + (1-p) \leq  -1 + (1-p)  \leq -p.
\end{eqnarray*}
Since every optimal policy must have a non-negative average reward, no policy in $\Strat^\epsilon$ is optimal for $c < c^*$.
\item
Now consider an optimal policy $\sigma^*$ in $\Strat^*$. We show that this policy also optimizes the probability of satisfaction of $\Aa$. 
There are two cases to consider.
\begin{enumerate}
    \item 
    If the expected average reward of $\sigma^*$ is $0$, then under no strategy it is possible to reach an accepting transitions (positive reward transition) in $\Mm \times \Rr_\Aa$. Hence, every policy is optimal in $\Mm$ against $\Aa$, and so is $\sigma^*$.
    \item 
    If the expected average reward of $\sigma^*$ is positive, then notice that for every BSCC of the Markov chain of $\Mm\times\Rr_\Aa$ under $\sigma^*$, the average reward is the same. This is so because otherwise, there is a positional policy that reaches the BSCC with the optimal average from all the other BSCCs with lower averages, contradicting the optimality of $\sigma^*$.
    Since for an optimal policy $\sigma^*$, every BSCC provides the same positive average, every BSCC must contain an accepting transition. 
    Hence, every run of the MDP $\Mm$ under $\sigma^*$ will eventually dwell in an accepting component and in the process will see a finitely many $\epsilon$ (reset) transitions. 
    For any such given run $r$, consider the the suffix $r'$ of the run after the last $\epsilon$ transition is taken and let $r = w r'$ for some finite run $w$. 
    Since $L(r')$ is an accepting word in $\Aa$, and since $\Aa$ is an absolute liveness property any arbitrary prefix $w'$ to this run $r'$ is also accepting. 
    This implies that the original run $r$ is also accepting for $\Aa$.
    It follows that for such a strategy $\sigma^*$, the probability of satisfaction of $\Aa$ is $1$, making $\sigma^*$ an optimal policy for $\Mm$ against $\Aa$. \qedhere
\end{enumerate}
\end{enumerate}
\end{proof}

Since our translation from absolute liveness $\omega$-regular objective to reward machines is model-free, the following theorem is immediate.
\begin{theorem}[Convergence of Model-free RL]
\label{thmmain}
  Differential $Q$-learning algorithm for maximizing average reward objective on
  $\Mm \times \Rr_\Aa$ will converge to a strategy maximizing the probability of
  satisfaction of $\Aa$ for a suitable value of $c$.
  Moreover, the product construction $\Mm\times \Rr_\Aa$ can be done on-the-fly and it is model-free.
\end{theorem}

As an example, consider the property $\eventually\always a$ and an MDP with two
states and all transitions between states are available as deterministic actions
(Fig.~\ref{fig:example_eventually1} and Fig.~\ref{fig:example_eventually2}). Only one of the states is labeled $a$. An
infinite memory strategy could see $a$ for one step, reset, then see two $a$s,
reset, then see three $a$s and so forth. This strategy will produce the same
average value as the positional strategy which sees $a$ forever without
resetting. However, the infinite memory strategy will fail the property while the
positional one will not. 

\subsubsection*{Shaping Rewards via Hard Resets.}
\label{sec:hardresets}
For a B\"uchi automaton $\Aa$, we say that its state $q \in Q$ is \emph{coaccessible} if there exists a path starting from that state to an accepting transition.
If a state is not coaccessible then any run of the product $\Mm\times \Aa$ that ends in such a state will never be  accepting, and hence one can safely redirect all of its outgoing (even better, incoming) transitions to the initial state with reward $c$ (a hard reset). 
Such hard resets will promote speedy learning by reducing the time spent in such states during 
unsuccessful explorations, and at the same time adding these resets does not make a non-accepting run accepting or vice versa. 
Lemma~\ref{l1}, Lemma~\ref{mm}, and Theorem~\ref{thmmain}  continue to hold with such hard resets.
Introducing hard resets is a reward shaping procedure in that it is a reward transformation~\cite{ng1999policy} under which optimal strategies remain invariant. 
\begin{figure}[t]
        \begin{tikzpicture}[
            >=latex,
            node distance=15mm,
            auto,
            thick,
            every state/.style={fill=yellow!20}
        ]
            \node[state, initial, initial text={}, initial where=left] (q0) {$q_0$};
            \node[state, accepting, right=of q0] (q1) {$q_1$};
            \node[state, right=of q1] (q2) {$q_2$};

            \path[->]
                (q0) edge node[above] {$a$} (q1)
                (q1) edge node[above] {$\neg a$} (q2)

                (q0) edge[loop above] node {$\top$} (q0)
                (q1) edge[loop above] node {$a$} (q1)
                (q2) edge[loop above] node {$\top$} (q2);
             \path[->, dashed, bend right] (q1) edge  (q0);
            \path[->, dashed, bend left] (q2) edge (q0);

        \end{tikzpicture}
        \caption{Automaton of $\eventually\always a$, dashed lines represent resets.}
        \label{fig:example_eventually1}
\end{figure}
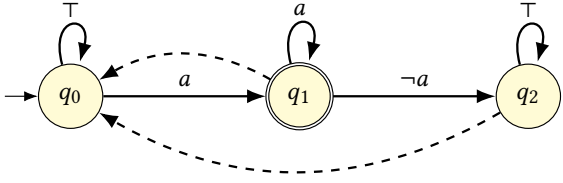

\begin{figure}
        \begin{tikzpicture}[
            >=latex,
            node distance=20mm,
            auto,
            thick,
            every state/.style={fill=yellow!20}
        ]
            \node[state] (q1) {$a$};
            \node[state,right=of q1] (q2) {$\neg a$};

            \path[->]
                (q1) edge[bend left] (q2)
                (q2) edge[bend left] (q1)
                (q1) edge[loop above] (q1)
                (q2) edge[loop above] (q2)
                ;
        \end{tikzpicture}
        \caption{MDP, each transition represents an action.}
        \label{fig:example_eventually2}
\end{figure}
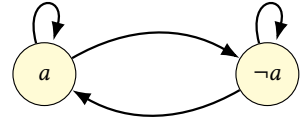

\subsubsection*{Extension to weakly communicating MDPs.}
To relax further the communicating assumption on the MDP, we can apply our method to weakly communicating MDPs but the class of specifications will be more restricted.
The set of weakly-communicating MDPs is the most general set of MDPs such that there currently exists a learning algorithm that can, using a single stream of experience, guarantee to identify a policy that achieves the optimal average reward rate in the MDP~\cite{bartlett2012regal,wan2022convergence}.
 We show how to apply our method in the following for weakly communicating MDPs and fairness specifications.

\begin{lemma}[Preservation of Weak Communication]
\label{lwc1}
For a weakly communicating MDP $\Mm$ and reward machine $\Rr_\Aa$ for
a fairness GFM automaton $\Aa$, the product $\Mm{\times}\Rr_\Aa$ is
weakly communicating. 
\end{lemma}

\begin{proof}
The proof directly follows the proof of Lemma~\ref{l1}. Since we have a weakly-communicating MDP we can partition the states of the MDP into two subsets: the first set is all the transient states under stationary policies, and the second set are the set of states where every pair of states are reachable from each other under  a stationary policy. Since the sub-MDP resulting from the second set is communicating, and a fairness property is also an absolute liveness property, then based on Lemma~\ref{l1}, the sub-product MDP is communicating. On the other hand, all of the states in the first set are transient, and the fairness property is closed under the addition and deletion of prefixes then eventually we reach a state of the communicating sub-product MDP which leads to the weakly communicating property of the product MDP. 
\end{proof}

%% file: multiobjective.tex
This section provides a solution for Problem~\ref{prob_lexo}.
We first note that an optimal policy for the Problem~\ref{prob_lexo} may require infinite memory. Consider the example in Figure~\ref{fig:infinite-memory}. The $\omega$-regular specification $\always\eventually a$ can be satisfied surely by staying in the state $a$. Likewise, the maximum possible average reward of $1$ can be obtained by staying in the state labeled $\neg a$. Surprisingly, there is an infinite memory policy that achieves both. This infinite memory policy visits the state labeled $\neg a$ $k$ times, then visits the state labeled $a$, where $k$ increases towards infinity forever. Any finite memory policy that visits $a$ infinitely often with probability $1$ must visit the state labeled $a$ with positive expected frequency in steady state, say $f > 0$, and thus achieves a suboptimal external average reward of at most $(1-f) < 1$. 
\begin{figure}%
    \centering
    \begin{tikzpicture}[
        >=latex,
        node distance=20mm,
        auto,
        thick,
        every state/.style={fill=yellow!20}
    ]
        \node[state] (q1) {$a$};
        \node[state, right=of q1] (q2) {$\neg a$};

        \path[->]
            (q1) edge[bend left] (q2)
            (q2) edge[bend left] (q1)
            (q1) edge[loop above] (q1)
            (q2) edge[loop above] node {$+1$} (q2);
    \end{tikzpicture}
    \caption{Example showing that the lexicographically optimal policy may require infinite memory. Actions are denoted as arrows; the reward is $0$ unless otherwise indicated. The specification of interest is $\always\eventually a$.}
    \label{fig:infinite-memory}
\end{figure}

However, one can find finite memory policies that achieve an average reward within $\varepsilon$ of the optimal average reward for any $\varepsilon > 0$ while maximizing the probability of satisfaction of the $\omega$-regular specification. Namely, consider the policy which visits the state labeled $\neg a$ $k$ times, then visits the state labeled $a$ for a fixed $k$. This policy obtains an average reward of $\frac{k}{k+2}$ while satisfying the $\omega$-regular specification with probability $1$. For any $\varepsilon > 0$, one can select $k \ge \lceil \frac{2}{\varepsilon} - 2 \rceil$ to ensure that $\frac{k}{k+2} \ge 1 - \varepsilon$.

In the following, we show that for any $\varepsilon>0$, there always exist a finite-memory policy that is $\varepsilon$-optimal for Problem~\ref{prob_lexo} and show how this policy can be constructed.
Our approach to solve Problem~\ref{prob_lexo} reduces the task of producing a lexicographically optimal strategy to one of solving a single average reward problem on an MDP, which can then be solved by any off-the-shelf RL algorithm for average reward on communicating MDPs, e.g. Differential Q-learning. The resulting policy is $\varepsilon$-optimal and finite memory. We then perform a transformation of this learned policy to yield an infinite memory optimal policy.

\begin{figure}[t]
\centering
\begin{minipage}[c]{0.58\textwidth}
\centering
\begin{tikzpicture}[
    >=latex,
    thick,
    every state/.style={draw,circle,fill=cyan!40,minimum size=7mm},
    gridline/.style={draw=gray!45, line width=0.25pt},
    gridframe/.style={draw=black!90, line width=0.9pt}
]

\def\cell{6mm}
\def\nx{6}
\def\ny{5}
\def\tilt{0.5}
\def\foresh{0.75}
\def\layersep{23mm}

\newcommand{\DrawGrid}[2]{%
    \draw[gridframe] (#1,#2) rectangle ++(\nx*\cell,\ny*\cell);
    \foreach \i in {1,...,5}{
        \draw[gridline] (#1+\i*\cell,#2) -- ++(0,\ny*\cell);
    }
    \foreach \j in {1,...,4}{
        \draw[gridline] (#1,#2+\j*\cell) -- ++(\nx*\cell,0);
    }
}

\begin{scope}[xslant=\tilt, yscale=\foresh, transform shape]
    \DrawGrid{0}{0}
    \node[anchor=east] at (-2mm, 1.0\cell) {$b=0$};

    \node[state] (a) at (0\cell, 0\cell) {$a$};
    \node[state] (b) at (1\cell, 1\cell) {$b$};
    \node[state] (c) at (2\cell, 2\cell) {$c$};
    \node[state] (e) at (0\cell, 2\cell) {$e$};
\end{scope}

\begin{scope}[yshift=-\layersep, xslant=\tilt, yscale=\foresh, transform shape]
    \DrawGrid{0}{0}
    \node[anchor=east] at (-2mm, 1.0\cell) {$b=1$};

    \node[state] (f) at (0\cell, 0\cell) {$f$};
    \node[state] (d) at (2\cell, 0\cell) {$d$};
    \node[state] (g) at (2\cell, 2\cell) {$g$};
    \node[state] (h) at (1\cell, 2\cell) {$h$};
\end{scope}

\path[->, very thick]
    (a) edge[loop right] (a)
    (a) edge[bend left] (b)
    (a) edge[bend left] (e)
    (b) edge[bend right] (c)
    (c) edge (e)
    (f) edge[bend left] (h)
    (d) edge[bend left] (f)
    (d) edge[bend left] (g)
    (f) edge (a);

\path[->, very thick, draw=black!60, opacity=0.35]
    (e) edge (h)
    (c) edge (g);

\end{tikzpicture}
\end{minipage}
\hfill
\begin{minipage}[c]{0.38\textwidth}
\centering
\begin{tikzpicture}[
    >=latex,
    node distance=30mm,
    auto,
    thick,
    every state/.style={draw,circle,fill=yellow!20}
]
    \node[state] (q1) {$b{=}0$};
    \node[state, right=of q1] (q2) {$b{=}1$};

    \path[->]
        (q1) edge[bend left] node {$\beta$} (q2)
        (q2) edge[bend left]  node {$(q, l, q') \in F$} (q1)
        (q1) edge[loop above] node {$1{-}\beta$}(q1)
        (q2) edge[loop above] node {$(q, l, q') \not \in F$} (q2);
\end{tikzpicture}
\end{minipage}

\caption{Picture of the construction of the probabilistic reward machine in \eqref{eq:prm}. \textbf{Left:} Two layers corresponding to $b\in\{0,1\}$. \textbf{Right:} Probabilistic changes of the additional bit $b$. The $\epsilon$-transitions are excluded for a better pictorial presentation.}
\label{fig:my_label}
\end{figure}

We define our average reward reduction via a reward machine. This reward machine augments the usual product with an additional bit. When this bit is zero, the agent receives reward based on the external reward function. With probability $\beta$ on every step, this bit flips to one.
When this bit is one, the agent is incentivized to find an accepting edge and is given a punishing reward of $\rho(s,s') + c_1$ where $c_1$ is a constant such that $c_1 + \max_{s,s'} \rho(s,s') < \min_{s,s'} \rho(s,s')$ on every step. 
When the agent finds an accepting edge, the bit flips back to zero. To ensure communication, we add two types of resets, one which resets the automaton state back to its initial state ($\epsilon_1$), and one which sets the extra bit to zero ($\epsilon_2$). Both of these incur some large penalty of $c_2$. We refer to the case when the bit is zero as ``the first layer'' and the case when the bit is one as ``the second layer''. Intuitively, the agent spends most of its time collecting the external average reward, but is occasionally called upon with probability $\beta$ to prove that it can see an accepting edge. This construction is shown pictorially in Figure~\ref{fig:my_label}.
We now define the reward machine formally. Let $\Aa = (\Sigma, Q, q_0,\delta, F)$ be a GFM for an absolute liveness specification and let $\rho(s, s')$ be a given reward function. We construct a probabilistic reward machine $\mathcal{R}_{\mathcal{A}\cdot\rho} = (\Sigma', Q\times \{0, 1\}, (q_0, 0), \delta', \rho')$ where $\Sigma' = (\Sigma \times S \times S \times Q)\cup\{\epsilon_1,\epsilon_2\}$,

\begin{align}
 \delta'((q, b),& (l, s, s', q'))((q', b'))
 =
\begin{cases}
\beta & b = 0, b' = 1, q' \in \delta(q,l), l\not\in\{\epsilon_1,\epsilon_2\}\\
1-\beta & b = 0, b' = 0, q' \in \delta(q,l), l\not\in\{\epsilon_1,\epsilon_2\}\\
1 & b = 1, b' = 1, q' \in \delta(q,l),(q,l,q') \not\in F, l\not\in\{\epsilon_1,\epsilon_2\}\\
1 & b = 1, b' = 0, q' \in \delta(q,l), (q,l,q') \in F, l\not\in\{\epsilon_1,\epsilon_2\}\\
1 & q \neq q_0, q' = q_0, b' = b, l = \epsilon_1 \\
1 & q' = q, b = 1, b' = 0, l = \epsilon_2 \\
0 & \text{otherwise}
\end{cases}
\label{eq:prm}
\end{align}
and
\[
\rho'((q, b), (l, s, s', q'), (q', b')) = 
\begin{cases}
\rho(s, s') & \text{if } b = 0 \text{ and } l\not\in\{\epsilon_1,\epsilon_2\}\\
c_1 + \rho(s, s') & \text{if } b = 1\text{ and } l\not\in\{\epsilon_1,\epsilon_2\}\\
c_2 & \text{if } l\in\{\epsilon_1,\epsilon_2\}
\end{cases}
\]
where $c_1 < \min_{s,s'} \rho(s,s') - \max_{s,s'} \rho(s,s')$, $c_2 < 0$ and $0 < \beta < 1$.

We now show a few results before proceeding to the main theorems.

\begin{lemma}[Preservation of Communication]
\label{lmo1}
For a communicating MDP $\mathcal{M}$ and reward machine $\mathcal{R}_{\mathcal{A}\cdot\rho}$ defined above, the resulting product $\mathcal{M}\times\mathcal{R}_{\mathcal{A}\cdot\rho}$ is communicating.
\end{lemma}
\begin{proof}
The claim holds due to the additions of $\epsilon$-transitions $\{\epsilon_1,\epsilon_2\}$ that resets both the automaton ($\epsilon_1$) and the extra bit $b$ to zero ($\epsilon_2$).
\end{proof}

\noindent To develop our formal proof in Theorem~\ref{theorem:epsilon-optimality}, we make use of the following observation.
\begin{lemma}[Binary Probability of Satisfaction]
\label{lemma:binary}
The probability of satisfaction of an absolute liveness specification is either $0$ or $1$ in a communicating MDP.

\end{lemma}
\begin{proof}
To show this result, we just need to show that if the probability of satisfaction of the specification is positive, then the probability of satisfaction must be $1$. This follows from the fact that if the probability of satisfaction is positive, then there is a winning end-component, and this end-component can be reached with probability $1$ due to the communicating nature of the product.
\end{proof}

\begin{theorem}[Lexicographic Reward Machine $\varepsilon$-optimality]
\label{theorem:epsilon-optimality}
Let $v^*$ be the external average reward obtained under a lexicographically optimal strategy for Problem~\ref{prob_lexo}. For every $\varepsilon > 0$, there exists a threshold $\beta^* > 0$ such that for all $0 < \beta < \beta^*$ there exists a $c^* < 0$ such that for all $c < c^*$ all positional policies that maximize the average reward on $\mathcal{M}\times\mathcal{R}_{\mathcal{A}\cdot\rho}$ maximize the probability of satisfaction of $\mathcal{A}$ on $\mathcal{M}$ and achieves an average reward $v$ such that $v \ge v^* - \varepsilon$. We call such a policy $\varepsilon$-optimal. Additionally, these policies are finite memory policies on $\mathcal{M}$.
\end{theorem}
\begin{proof}
We first note that since the memory required in the resulting policies is defined by $\mathcal{R}_{\mathcal{A}\cdot\rho}$ with fixed parameters, the resulting policies are finite memory.
From Lemma~\ref{lemma:binary}, we can proceed by considering two cases: the absolute liveness $\omega$-regular objective is satisfied with probability $0$ or $1$.
\begin{enumerate}
    \item 
    We first consider the case where the absolute liveness $\omega$-regular objective is satisfied with probability $0$ under any policy, i.e., there is no policy that can satisfy the specification with positive probability. We will show that any $0 < \beta^* < 1$ and $c^* < 0$ works. Since the lexicographically optimal policy only needs to maximize the external average reward, there is an optimal policy that is positional on $\mathcal{M}$ that achieves an external average reward $v^*$. 
    Consider an optimal policy $\sigma$ on the extended product $\mathcal{M}\times\mathcal{R}_{\mathcal{A}\cdot\rho}$. In steady state, this policy either takes resets from the second layer to the first infinitely often, or it takes no reset transitions. If the policy takes no resets, then it obtains a value of $R_\sigma = b + v \le b + v^*$ where the upper bound $b + v^*$ is obtained by applying a positional policy optimal for the external reward to the second layer. Since $\sigma$ is optimal, it must obtain a value of $b + v^*$, which obtains an external average reward of $v^*$ on $\mathcal{M}$. If this policy takes resets infinitely often, then it takes the reset transition immediately upon reaching the second layer. This is due to the penalty $b$: if a trajectory of length $T$ in the second layer that ends in a reset obtains a value of $T (b + v) + c$, then taking the reset first, and followed by the same trajectory in the first layer obtains a value of $T v + c \ge T (b + v) + c$, where $v$ is the average reward collected along this trajectory. Thus, this policy obtains a value of $R_\sigma = (1-\beta) v_1 + \beta c$ in the extended product, which is maximized when $v_1 = v^*$. Since the resets are internal to the agent, this policy collects an external average reward of $v^*$ on $\mathcal{M}$. We have shown that under any $0 < \beta^* < 1$ and $c^* < 0$ when the absolute liveness $\omega$-regular objective is satisfied with probability $0$, an optimal policy on the extended product collects an external average reward of $v^* \ge v^* - \varepsilon$ for every $\varepsilon > 0$.
    \item We now consider the case where the absolute liveness $\omega$-regular objective is satisfied with probability $1$ under some policy. Let $G \subseteq (\mathcal{M} \times \mathcal{A}) \times (\mathcal{M} \times \mathcal{A})$ be the set of transitions such that for every $t \in G$, there exists a policy that satisfies $\mathcal{A}$ with probability $1$ and visits $t$ infinitely often. In other words, $G$ is the maximal set of transitions that can be visited infinitely often without diminishing the probability of satisfying $\mathcal{A}$. Consider a lexicographically optimal strategy. This strategy cannot visit any transition outside of $G$ infinitely often by definition and achieves an external average reward of $v^*$. This implies that there exists a set of positional policies restricted to $G$ that achieves an average reward of at least $v^*$, since there exists optimal positional policies for average reward on MDPs. We now extend these policies to the extended product by applying this policy in the first layer and applying a policy which visits an accepting edge in the second layer without taking resets infinitely often and leaves $G$ finitely many times, with probability $1$. This second layer policy exists due to the definition of $G$. We call this set of policies $\sigma^*$, and abuse notion by treating $\sigma^*$ as if it is a single policy for the rest of the proof. We will select $0 < \beta^* < 1$ and $c^* < 0$ such that $\sigma^*$ is optimal in the extended product, and is $\varepsilon$-optimal on $\mathcal{M}$. Consider all policies $\Pi_1$ that are positional on the extended product, satisfy the specification with probability $1$, and visit the first layer infinitely often with probability $1$. We say that the average reward that such a policy $\sigma \in \Pi_1$ obtains on the extended product for a particular $\beta$ is $R_\sigma(\beta)$. Let $v_\sigma$ be the external average reward that $\sigma$ obtains when $\beta = 0$. Then, we have that $(1-\beta)v_\sigma + \beta b f_\sigma(\beta) \le R_\sigma(\beta) \le v_\sigma$ where $f_\sigma(\beta)$ is an upper bound on the expected time it takes to return to a state in steady state after a transition from the first to the second layer occurs. The first inequality follows by breaking up the average reward into uninterrupted reward collected in the first layer in steady state, and reward collected while waiting to return to the state that was left by transitioning to the second layer in steady state. The second inequality follows from the fact that the reward in the second layer is always less than the first. Note that $f_\sigma(\beta)$ is monotonically increasing in $\beta$. Consider a policy $\sigma \in \Pi_1$ such that $v_{\sigma} < v_{\sigma^*} = v^*$. Then we have that for $\beta \le \min\{\frac{1}{2} , \frac{v^* - v_\sigma}{v^* - bf_{\sigma^*}(\frac{1}{2})}\}$, $$R_\sigma(\beta) \le v_\sigma \le (1-\beta)v_{\sigma^*} + \beta b f_{\sigma^*}(\beta) \le (1-\beta)v_{\sigma^*} + \beta b f_{\sigma^*}(\frac{1}{2}) \le R_{\sigma^*}(\beta)$$
    where the third inequality follows from the monotonicity of $f_{\sigma^*}(\beta)$. Since there are only finitely many policies in $\Pi_1$, one can find a $0 < \beta_{\text{thresh}} < 1$ that satisfies the above inequality for all policies in $\Pi_1$. Thus, for all $\beta < \beta_\text{thresh}$, $\sigma^*$ is an optimal average reward policy on the extended product amongst all the policies in $\Pi_1$. Note that the external average reward collected on $\mathcal{M}$ by any policy $\sigma$ for a fixed $\beta$, denoted by $\bar R_\sigma(\beta)$, is larger than the average reward collected on the extended product $R_\sigma(\beta)$ because all of the rewards on the extended product are less than or equal to the external rewards. By selecting $\beta^* = \min\{\beta_\text{thresh}, \frac{\varepsilon}{v^* - bf_{\sigma^*}(\frac{1}{2})}\} > 0$, we have that
    $$v^* - \varepsilon \le (1-\beta)v_{\sigma^*} + \beta b f_{\sigma^*}(\frac{1}{2}) \le R_{\sigma^*}(\beta) \le \bar R_{\sigma^*}(\beta),$$
    so $\sigma^*$ is $\varepsilon$-optimal for all $0 < \beta < \beta^*$.
    We will now select $c^* < 0$ such that there is an optimal policy in $\Pi_1$. We denote all positional policies by $\Pi$. Note that all policies in $\Pi \backslash \Pi_1$ either remain in the second layer forever, obtaining an average reward of $b < \min_{s,s'} \rho(s,s')$ which is less than the average reward for all policies $\Pi_1$, or take reset transitions infinitely often. Policies which select reset transitions infinitely often obtain an average reward of at most $(1-\beta)\max_{s,s'}\rho(s,s') + \beta c$. One can then select $c^* \le \frac{b}{\beta} - \frac{1-\beta}{\beta} \max_{s,s'} \rho(s,s') < 0$ to ensure that $(1-\beta)\max_{s,s'}\rho(s,s') + \beta c \le b$. Since all policies in $\Pi_1$ obtain an average reward on the extended product greater than $b$, there is an optimal policy $\sigma^*$ in $\Pi_1$ for $c < c^*$ with the previously fixed $\beta$. This $\sigma^*$ is an optimal policy on the extended product under these parameter selections and is $\varepsilon$-optimal when applied to $\mathcal{M}$.
\end{enumerate}
\end{proof}

\begin{theorem}[Lexicographic Reward Machine Optimality]
\label{theorem:lexo}
Let $\beta(i): \mathbb{N} \to (0,1)$ be a sequence such that $\lim_{i\to\infty} \beta(i) = 0$. There exists a threshold $\beta^* > 0$ such that for all $0 < \beta < \beta^*$ there exists a $c^* < 0$ such that for all $c < c^*$ all positional policies that maximize the average reward on $\mathcal{M}\times\mathcal{R}_{\mathcal{A}\cdot\rho}$ have the following property: if we transform this fixed, finite memory policy by setting $\beta = \beta(i)$ on timestep $i$, then this transformed policy is optimal for Problem~\ref{prob_lexo} and infinite memory.
\end{theorem}
\begin{proof}
We first note that this policy is infinite memory: it must keep track of $\beta(i)$, which requires an increasing number bits to represent exactly as $\beta(i)$ decreases. Thus, as $i$ goes to infinity, the amount of memory required grows in an unbounded manner.
For the rest of the proof, we follow the results and proof of Theorem~\ref{theorem:epsilon-optimality}. Let $\beta^* = \beta_\text{thresh}$ as defined in the proof of Theorem~\ref{theorem:epsilon-optimality}, and let 
\[
c^* \le \frac{b}{\beta} - \frac{1-\beta}{\beta} \max_{s,s'} \rho(s,s') < 0.
\]
We have that there is a set of optimal positional policies $\sigma^*$ on the extended product that satisfy $\mathcal{A}$ with the maximum probability and accumulate an external average reward $\bar R_{\sigma^*}(\beta) \ge (1-\beta)v^* + \beta b f_{\sigma^*}(\frac{1}{2})$ where $v^*$ is the external average reward obtained under a lexicographically optimal strategy. We now substitute $\beta$ with $\beta(i)$. We have that 
$$\lim_{i\to\infty} \bar R_{\sigma^*}(\beta(i)) \ge \lim_{i\to\infty} (1-\beta(i))v^* + \beta(i) b f_{\sigma^*}(\frac{1}{2}) = v^* ,$$
so $\sigma^*$ is a lexicographically optimal policy.
\end{proof}

These results can be extended to weakly communicating MDPs in straightforward fashion.
\begin{lemma}[Preservation of Communication]
\label{lmowc1}
For a weakly-communicating MDP $\mathcal{M}$ and reward machine $\mathcal{R}_{\mathcal{A}\cdot\rho}$ over a fairness specification, the resulting product $\mathcal{M}\times\mathcal{R}_{\mathcal{A}\cdot\rho}$ is weakly-communicating.
\end{lemma}
\begin{proof}
The proof follows directly from Lemma~\ref{lwc1} and Lemma~\ref{lmo1}. As stated for Lemma~\ref{lwc1} we can partition the states of the MDP into two sets and the sub-product MDP resulted from the second set with the fairness specification is communicating based on Lemma~\ref{lmo1}. Moreover, since all of the states of the first set, see the proof of Lemma~\ref{lwc1}, are transient the resulting product MDP is weakly communicating.
\end{proof}

%% file: exp.tex
We implemented the reduction\footnote{The implementation is available at \href{https://plv.colorado.edu/mungojerrie/}{https://plv.colorado.edu/mungojerrie/}.} with hard resets presented in Section~\ref{sec:result}. 
As described, we do not build the product MDP explicitly, and instead compose it on-the-fly by keeping track of the MDP and automaton states independently. We use Differential Q-learning~\cite{wan2021learning} to learn optimal, positional average reward strategies. For our experiments, we have collected a set of communicating MDPs with absolute liveness specifications.\footnote{Case studies are available at \href{https://doi.org/10.5281/zenodo.15489871}{https://doi.org/10.5281/zenodo.15489871}.}
All experiments were conducted on 64-bit machine equipped with a 12th Gen Intel Core i7 processor, 32 GB of RAM, and an NVIDIA A2000 GPU.

We compare with two previous approaches for translating absolute liveness $\omega$-regular languages to rewards: the method of~\cite{hahn2019omega} with Q-learning and the method of~\cite{bozkurt2019control} with Q-learning. The method of~\cite{hahn2019omega} translates a GFM B\"uchi automaton into a reachability problem through a suitable parameter $\zeta$. This reachability problem can be solved with discounted RL by rewarding reaching the target state and using a large enough discount factor. The method of~\cite{bozkurt2019control} uses a state dependent discount factor $\gamma_B$ and a GFM B\"uchi automaton. By using a suitable $\gamma_B$ and large enough discount factor, one can learn optimal strategies for the absolute liveness $\omega$-regular objective.

\subsubsection*{RQ1. How do previous approaches perform in the continuing setting?}
The methods of~\cite{hahn2019omega,bozkurt2019control} may produce product MDPs that are not communicating (see Example~\ref{example:noncommunication}). This means that a single continuing run of the MDP may not explore all relevant states and actions. Thus, these methods are not guaranteed to converge without episodic resetting.
We studied if this behavior affects these prior methods in practice. As a baseline, we include our proposed approach. Instead of tuning hyperparameters for each method, where hyperparameters that lead to convergence may not exist, we take a sampling approach. We select a wide distribution over hyperparameters for each method and sample $200$ hyperparameter combinations for each method and example. We then train for $10$ million steps on each combination. The selected hyperparameter distribution is $\alpha \sim \mathcal{D}(0.01, 0.5)$, $\varepsilon \sim \mathcal{D}(0.01, 1.0)$, $c \sim \mathcal{D}(1, 200)$, $\eta \sim \mathcal{D}(0.01, 0.5)$, $\zeta \sim \mathcal{D}(0.5, 0.995)$, $\gamma_B \sim \mathcal{D}(0.5, 0.995)$, and $\gamma \sim \mathcal{D}(0.99, 0.99999)$ where $\mathcal{D}(a,b)$ is a log-uniform distribution from $a$ to $b$. The end points of these distributions and the training amount are selected by finding hyperparameters which lead to convergence in the episodic setting for these methods.

Figure~\ref{fig:continuing} shows the resulting distribution over runs. A distribution entirely at $0$ ($1$) indicates that all sampled runs produced strategies that satisfy the specification with probability $0$ ($1)$. 
\begin{figure}[tb]
\centering
\begin{tikzpicture}
    \node[anchor=south west, inner sep=0] (image) at (0,0) {
        \includegraphics[scale=0.62]{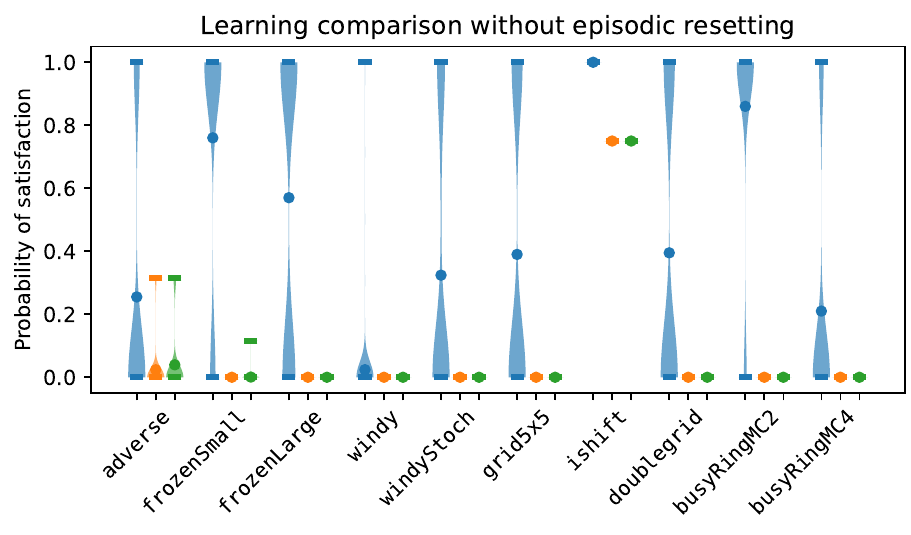}
    };

    \node [
        draw=gray!50,           %
        fill=white,             %
        anchor=north west,      %
        xshift=5pt,             %
        yshift=1.3cm,%
        inner sep=3pt           %
    ] at (image.east) {   %
        \footnotesize           %
        \begin{tabular}{@{}c l@{}} %
            \textcolor{mplblue}{\rule[-1pt]{8pt}{8pt}} & Average (ours) \\
            \textcolor{mplorange}{\rule[-1pt]{8pt}{8pt}} & Hahn et al.~\cite{hahn2019omega} \\
            \textcolor{mplgreen}{\rule[-1pt]{8pt}{8pt}} & Bozkurt et al.~\cite{bozkurt2019control} \\
        \end{tabular}
    };
\end{tikzpicture}
\caption{Comparison of the distributions of probability of satisfaction of learned policies across sampled hyperparameters in the continuing setting. For each distribution, the mean is shown as a circle, and the maximum and minimum are shown as vertical bars. We compare our proposed reduction, the reduction of Hahn et al.~\cite{hahn2019omega} with Q-learning, and the reduction of Bozkurt et al.~\cite{bozkurt2019control} with Q-learning. Episodic resetting is not used.}
\label{fig:continuing}
\end{figure}
\emph{For many examples, prior approaches have no successful hyperparameter combinations, with distributions centered entirely at $0$.} However, our proposed approach always have some hyperparameters that lead to optimal, probability $1$, strategies, as indicated by the tails of the distributions touching the probability $1$ region of the plot.

\subsubsection*{RQ2. How does our method compare to previous approaches when we allow episodic setting?}
By allowing episodic resetting, we can now find hyperparameters for methods of \cite{hahn2019omega,bozkurt2019control} that lead to convergence. We tuned all hyperparameters by hand to minimize training time, while verifying with a model checker that the produced policies are optimal. Table~\ref{tab:experiment} shows learning times, as well as hyperparameters for our reduction. We report the number of states reachable in the MDP and the product, learning times averaged over $5$ runs, the reset penalty $c$, the $\varepsilon$-greedy exploration rate $\varepsilon$, the learning rates $\alpha$ and $\eta$ of the Differential Q-learning, as well as the number of training steps. Note that we do not do any episodic resetting when training with our reduction. This means that the RL agent must learn to recover from mistakes during training, while previous approaches are periodically reset to a good initial state. \emph{Our reduction using Differential Q-learning is competitive with previous approaches while not being reliant on episodic resetting.}

\begin{table*}
	\caption{Learning results and comparison. Hyperparameters used for our reduction are shown. Blank entries indicate that default values are used. The default parameters are $c=-1$, $\varepsilon=0.1$, $\alpha=0.1$, and $\eta=0.1$. Times are in seconds. Superscript~$^\dagger$ indicates results from Q-learning with reduction from~\cite{hahn2019omega}, while superscript $^\ddagger$ indicates Q-learning with reduction from~\cite{bozkurt2019control}. Results for $^\dagger$ and $^\ddagger$ require episodic resetting. All hyperparameters are tuned by hand.}
	\label{tab:experiment}
    \centering
	\footnotesize
	\setlength{\tabcolsep}{4pt}
	\begin{tabular}[c]{l|ccccc|ccccc}
		Name & states & prod. & time & time$^\dagger$ & time$^\ddagger$ & $c$ & $\varepsilon$ & $\alpha$ & $\eta$ & train-steps \\\hline
		\texttt{adverse}    & 202 & 507 & 8.51 & 7.09 & 12.56 & -150 &  & 0.2 &  & 10M   \\
		\texttt{frozenSmall}    & 16 & 64 & 0.99 & 20.23 & 9.88 &  &  &  &  & 500k   \\
		\texttt{frozenLarge}    & 64 & 256 & 4.07 & 3.88 & 8.79 &  &  & 0.02 & 0.02 & 3M   \\
		\texttt{windy}    & 123 & 366 & 1.40 & 1.81 & 2.61 &  & 0.95 & 0.5 & 0.05 & 1M   \\
		\texttt{windyStoch}    & 130 & 390 & 2.97 & 3.91 & 2.53 &  &  & 0.5 &   & 2M   \\
		\texttt{grid5x5}    & 25 & 100 & 0.62 & 1.12 & 1.02 &  &  & 0.5 &  & 200k   \\
		\texttt{ishift}    & 4 & 29 & 0.03 & 0.01 & 0.02 &  &  &  &  & 10k   \\
		\texttt{doublegrid}    & 1296 & 5183 & 16.43 & 3.45 & 3.09 & -2 & 0.5 & 0.05 & 0.01 & 12M   \\
		\texttt{busyRingMC2}    & 72 & 288 & 0.03 & 0.03 & 0.03 &  &  &  & 0.01 & 10k   \\
		\texttt{busyRingMC4}    & 2592 & 15426 & 6.08 & 3.94 & 2.33 &  &  &  & 0.01 & 1.5M\\\hline
	\end{tabular}

\end{table*}

\begin{table*}[t]
	\caption{Learning results for multi-objective case studies. The default parameters are  $\texttt{tol}= 0$, $\beta=0.05$, $\texttt{ep-n}= 1$, $\texttt{ep-l}=1000000$, $\alpha=0.01$, $\eta=0.01$, $\varepsilon=0.1$. Times are in seconds. All hyperparameters are tuned by hand.}
	\label{tab:experimentmulti}
	\centering
	\footnotesize
	\setlength{\tabcolsep}{4pt}
	\begin{tabular}[c]{l|ccccc|c}
		Name & states & prod. & time & Prob. & rew.&  $c$ \\\hline
		\texttt{infmem}    & 2 & 8 & 0.758 & 1 & 0.999&100\\
		\texttt{vessels}    & 24 & 576 & 1.115 & 1 & -0.100 &  50\\
		\texttt{grid4x4c4}    & 16 & 512 & 1.959 & 1 & 0.004 &  50\\
        \texttt{grid4x4c2}    & 16 & 256 & 1.930 & 1 & 0.002 &  50\\
		\texttt{grid5x5c4}    & 25 & 800 & 1.977 & 1 & -0.001 &  60\\
		\texttt{grid5x5c2}    & 25 & 400 & 1.872 & 1 & -0.001 &  60\\
		\texttt{grid4x7}    & 28 & 896 & 1.741 &  1 & 0.000 & 100\\   
		\hline
	\end{tabular}
\end{table*}

\subsubsection*{RQ3. How our approach works in multi-objective settings?} As shown in Table~\ref{tab:experimentmulti}, we successfully satisfied the absolute liveness property with probability one across all case studies. We report the number of states reachable in the MDP and the product, the learning times averaged over five runs, the average reward over five runs, the reset penalty $c$ and $\beta$, the $\varepsilon$-greedy exploration rate $\varepsilon$, the learning rates $\alpha$ and $\eta$ of Differential Q-learning, as well as the number of training steps. 
Note that, also in the multi-objective setting, we do not perform any episodic resetting when training with our reduction. We achieved average rewards that are almost optimal in all cases (the optimal value for the first case study is $1$ and for the rest is $0$).

\input{related_work}

%% file: related_work.tex
\section{Related Work}
\label{sec:related-work}

\subsubsection*{Learning for Formal Specifications over Finite Traces.} The development and use of formal reward structures for RL have witnessed increased interest in recent years. For episodic RL, logics have been developed over finite traces of the agent's behavior, including LTL$_f$ and Linear Dynamic Logic (LDL$_f$) \cite{LTLf2, LTLf3}. These logics have equivalent automaton and reward machine representations that have catalyzed a series of efforts on defining novel reward shaping functions to accelerate the convergence of RL algorithms subject to formal specifications \cite{RL_LTL_f, icarte2018using, camacho2019ltl}. These methods leverage the graph structure of the automaton to provide an artificial reward signal to the agent. More recently, dynamic reward shaping using LTL$_f$ has been introduced as a means to both learn the transition values of a given reward machine and leverage these values for reward shaping and transfer learning \cite{velasquez2021dynamic}. There has also been work on learning or synthesizing the entire structure of such reward machines from agent interactions with the environment by leveraging techniques from satisfiability and active grammatical inference \cite{adviceSAT, xu2020SAT, gaon2020reinforcement, XuWuNeiderTopcu21, toro2019learning}.
Other data-driven methods have also been developed recently for satisfying specifications over finite traces. Data-driven distributionally robust policy synthesis approaches are presented in \cite{kordabad2025data,schon2024data}. Data-driven construction of abstractions with correctness guarantees has also been studied \cite{nazeri2025data,schon2024data_binary,kazemi2024data}, which enables policy synthesis against temporal specifications. Data-driven computation of resilience with respect to finite temporal logic specifications is studied in \cite{saoud2024temporal}.

\subsubsection*{Formal Specifications over Infinite Traces.} For the infinite-trace settings, LTL has been extensively used to verify specifications and synthesize policies formally using the mathematical model of a system \cite{BK08,BL06,majumdar2019symbolic,majumdar2024necessary,kazemi2020formal,lavaei2020formal,Sriram_Predictive19,HS_TAC20,HNS2021,schon2022correct,Lavaei_Survey,kazemi2024assume}. Considering the generality of the available results in terms of structure of the underlying MDP, most of the research focuses on discounted reward structures. Despite the simplicity of discounted Markov decision problems, the discounted reward structure (unlike average reward) prioritizes the transient response of the system. 
However, application of the average reward objective because of the restriction over the structure of the MDP is limited. The work \cite{ding2014optimal} proposes a policy iteration algorithm for satisfying specifications of the form $\always\eventually \phi\wedge \psi$ for a communicating MDP almost surely. 
The work \cite{ashok2017value} proposes a value iteration algorithm for solving the average reward problem for multichain MDPs, where the algorithm first computes the optimal value for each of strongly connected components and then weighted reachability to find the optimal policy.
The work \cite{atia2021steady} provides a linear program for policy synthesis of multichain MDPs with steady-state constraints. The work \cite{majumdar2024regret} provides a regret-free learning algorithm for policy synthesis that is based on identifying the structure of the underlying MDP using data with a certain confidence. 

\subsubsection*{RL for Formal Specifications over Infinite Traces.} In the last few years, researchers have started developing data-driven policy synthesis techniques in order to satisfy temporal specifications.
There is a large body of literature in safe reinforcement learning (RL) (see e.g. \citep{garcia2015comprehensive,recht2018tour, efroni2020exploration}). The problem of learning a policy to maximize the satisfaction probability of a temporal specification using discounted RL is studied recently \cite{brazdil2014verification,fu2014probably,sadigh2014learning, bozkurt2019control, hasanbeig2019certified, hasanbeig2019reinforcement, oura2020reinforcement}. 
The work \cite{hahn2019omega} uses a parameterized augmented MDP to provide an RL-based policy synthesis for finite MDPs with unknown transition probabilities. 
It shows that the optimal policy obtained by RL for the reachability probability on the augmented MDP gives a policy for the MDP with a suitable convergence guarantee  

In \cite{bozkurt2019control}
authors provide a path-dependent discounting mechanism for the RL algorithm based on a limit-deterministic B\"uchi automaton (LDBA) representation of the underlying $\omega$-regular specification, and prove convergence of their approach on finite MDPs when the discounting factor goes to one.
An LDBA is also leveraged in \cite{hasanbeig2019certified, hasanbeig2019reinforcement,oura2020reinforcement} for discounted-reward model-free RL in both continuous- and discrete-state MDPs. The LDBA is used to define a reward function that incentivizes the agent to visit all accepting components of the automaton.
The paper \cite{falah2023reinforcement} proposes a translation from $\omega$-regular properties into discounted reward objectives for model-free RL in continuous-time MDPs utilizing similar automata-based reward functions.

These works use episodic discounted RL with discount factor close to one to solve the policy synthesis problem. There are two issues with the foregoing approaches. First, because of the episodic nature of the algorithms they are not applicable in continuing settings. Second, because of high discount factors in practice these algorithm are difficult to converge. On the other hand, recent work on reward shaping for average reward RL has been explored based on safety specifications to be satisfied by the synthesized policy \cite{jiang2021temporal}. In contrast to the solution proposed in this paper, the preceding approach requires knowledge of the graph structure of the underlying MDP and does not account for absolute liveness specifications.

\subsubsection*{Average-Reward RL.} There is a rich history of studies in average reward RL \cite{dewanto2020average, mahadevan1996average, tadepalli1998model}. 
Lack of stopping criteria for multichain MDPs affect the generality of model-free RL algorithms. 
Therefore, all model-free RL algorithms put some restrictions on the structure of MDP (e.g. ergodicity \cite{pmlr-v97-lazic19a, wei2020model} or communicating property). The closest line of work to this work is to use average reward objective for safe RL. The work \cite{singh2020learning} proposes a model-based RL algorithm for maximizing average reward objective with safety constraint for communicating MDPs.
It is worth noting that in multichain setting, the state-of-the-art learning algorithms use model-based RL algorithms.  The work \cite{kretinsky2020finite} studies satisfaction of $\omega$-regular specifications using data-driven approaches. The authors introduce an algorithm  where the optimality of the policy is conditioned to not leaving the corresponding maximal end component which leads to a sub-optimal solution. The authors provide PAC analysis for the algorithm as well. Finally, there is a growing area of research on robust average-reward RL, whereby the transition probability function is defined over an uncertainty set, e.g., a KL divergence set \cite{wang2023model, wang2024robust}. The problem of finding an optimal policy in that context is to maximize the minimum expect reward within that uncertainty set of transition probabilities.

\subsubsection*{Multi-objective Probabilistic Verification.}
Chatterjee, Majumadar, and Henzinger~\cite{chatterjee2006markov} considered MDPs with multiple discounted reward objectives. 
In the presence of multiple objectives, the trade-offs between different objectives can be characterized as Pareto curves. 
The authors of~\cite{chatterjee2006markov} showed that every Pareto optimal point can be achieved by a memoryless strategy and the Pareto curve can be approximated in polynomial time. 
Moreover, the problem of checking the existence of a strategy that realizes a value vector can be decided in polynomial time. 
These multi-objective optimization problems were studied in the context of multiple long-run average objectives by Chatterjee~\cite{Chatte07}. 
He showed that the Pareto curve can be approximated in polynomial time in the size of the MDP for irreducible MDPs and in polynomial space in the size of the MDP for general MDPs.  
Additionally, the problem of checking the existence of a strategy that
guarantees values for different objectives to be equal to a given vector is in polynomial time for irreducible MDPs and in NP for general MDPs. 
Etessami et al.~\cite{EKVY07} were the first to study the multi-objective 
model-checking problem for MDPs with $\omega$-regular objectives.
Given probability intervals for the satisfaction of various properties, 
they developed a polynomial-time (in the size of the MDP) algorithm to decide
the existence of such a strategy.
They also showed that, in general, such strategies may require both randomization and memory. 
That paper also studies the approximation of the Pareto curve with respect to a
set of $\omega$-regular properties in time polynomial in the size of the MDP.
Forejt et al.~\cite{FKNPQ11} studied quantitative multi-objective optimization over MDPs that combines 
$\omega$-regular and quantitative objectives.
Those algorithms are implemented in the probabilistic model checker PRISM~\cite{Kwiatk11}. Multi-objective optimization on interval MDPs is studied by Monir et al.~\cite{monir2024lyapunov} to provide a Lyapunov-based policy synthesis approach.

\subsubsection*{Multi-objective Reinforcement Learning.}
There has been substantial work on lexicographic objectives in RL, including lexicographic discounted objectives~\cite{skalse2022lexicographic}, lexicographic $\omega$-regular objectives~\cite{hahn2021model}, and a combination of safety and discounted objectives~\cite{Bozkurt0P21}.
Hahn et al.~\cite{hahn2023omega} considered the general class of $\omega$-regular objectives with discounted rewards in model-free RL. However, to the best of our knowledge, this paper is the first work that considers lexicographic 
$\omega$-regular objectives with average objectives. 

\subsubsection*{Average-Reward RL for Formal Specifications.} Despite significant progress in data-driven approaches for satisfying $\omega$-regular specifications, there remains a gap in the use of average-reward, model-free RL algorithms for satisfying temporal logic specifications. While the paper \cite{alur2022framework} proves non-existence of robust, optimality-preserving reductions from general LTL specifications to discounted cumulative rewards, it identifies the reduction to average rewards as an open problem. This open problem has been addressed in the paper \cite{le2024reinforcement} that shows that one can translate any specification in the general class of $\omega$-regular properties into average objectives and then approximate them successively with discounted objectives while preserving optimality in the limit in the episodic setting.
Our preliminary work presented at AAMAS 2022~\cite{kazemi2022translating} also aimed to address this gap without relying on successive discounted RL in episodic setting but  by proposing a model-free average-reward RL algorithm tailored to a subclass of LTL specifications known as absolute liveness specifications. We argue that this subclass captures a broad range of practically relevant specifications and is particularly well-suited to the average-reward RL setting.

This manuscript builds on~\cite{kazemi2022translating} and extends the results in the following directions:
\textbf{(a)} we generalize our results from communicating MDPs to the broader class of weakly communicating MDPs;
\textbf{(b)} we incorporate a lexicographic multi-objective optimization framework, wherein the goal is to maximize a mean-payoff objective among policies that satisfy a given liveness specification;
\textbf{(c)} we provide novel automata-theoretic characterizations of absolute liveness and stable specifications; and
\textbf{(d)} we conduct an experimental evaluation demonstrating the effectiveness of the proposed lexicographic multi-objective approach.

%% file: conclusion.tex
This work addressed the problem of synthesizing policies that satisfy a given absolute liveness $\omega$-regular specification while maximizing a mean-payoff objective in the continuing setting. Our first key contribution is a \emph{model-free translation} from $\omega$-regular specifications to an average-reward objective, enabling the use of off-the-shelf average-reward reinforcement learning algorithms. In contrast to existing approaches that rely on discounted, episodic learning, which require environment resets and may be infeasible in many real-world settings, our approach learns optimal policies in a single, life-long episode without resetting.

Our second contribution is a solution to a \emph{lexicographic multi-objective optimization problem}, where the goal is to maximize a mean-payoff objective among the set of policies that satisfy a given liveness specification. We provide convergence guarantees under the assumption that the underlying Markov Decision Process is communicating.
Importantly, our solutions are \emph{model-free} and do not require access to the environment's transition structure or its graph representation. This removes the common assumption in prior work that synthesis requires the computation of end components in the product MDP.

We implemented our approach using Differential Q-learning and evaluated it on a range of case studies. The experimental results demonstrate that our method reliably converges to optimal strategies under the stated assumptions and outperforms existing approaches in the continuing setting. These results support the important and often overlooked hypothesis that \emph{average-reward RL is better suited to continuing tasks than discounted RL}.
As future work, we plan to explore the use of \emph{function approximation}, with the aim of enabling average-reward RL to achieve the same level of success for continuing tasks as discounted RL has achieved in episodic settings.